\documentclass[10pt,journal]{IEEEtran}
\usepackage{epsfig}
\usepackage{float}
\usepackage{graphicx}
\usepackage{amsmath}
\usepackage{amssymb}
\usepackage{amsthm}
\usepackage[linesnumbered,ruled]{algorithm2e}
\usepackage{colortbl,booktabs}
\usepackage{threeparttable}
\usepackage{subeqnarray}
\usepackage{cases}
\usepackage{multirow}
\usepackage{array}
\usepackage[table]{xcolor}
\usepackage{bm}
\usepackage[switch]{lineno}
\usepackage{threeparttable}
\usepackage{fixltx2e}
\usepackage{enumerate}
\usepackage{epstopdf}
\usepackage[hyphenbreaks]{breakurl}
\usepackage[hyphens]{url}
\usepackage{indentfirst}
\usepackage{color}
\newcolumntype{I}{!{\vrule width 2pt}}
\newlength\savedwidth

\newlength\savewidth

\newtheorem{theorem}{Theorem}

\newtheorem{proposition}[theorem]{Proposition}

\theoremstyle{definition}

\ifCLASSINFOpdf
\else
\fi
\begin{document}

\title{Domain Adaptation by Class Centroid \\ Matching and Local Manifold Self-Learning}
%
%
%

\author{Lei Tian*,~
        Yongqiang Tang*,~
		Liangchen Hu,~
        Zhida Ren,~
		and Wensheng~Zhang~
\thanks{* indicates equal contributions.}
\thanks{L. Tian, Z. Ren and W. Zhang are with the Research Center of Precision Sensing and Control, Institute of Automation, Chinese Academy of Sciences, Beijing, 100190, China, and University of Chinese Academy of Sciences, Beijing, 101408, China.
E-mail:\{tianlei2017, renzhida2017\}@ia.ac.cn, zhangwenshengia@hotmail.com.}
\thanks{Y. Tang is with the Research Center of Precision Sensing and Control, Institute of Automation, Chinese Academy of Sciences, Beijing, 100190, China.
E-mail:tangyongqiang2014@ia.ac.cn}
\thanks{L. Hu is with the School of Computer Science and Engineering, Nanjing University of Science and Technology, Nanjing, 210094, China.
	E-mail:hlc\_clear@foxmail.com.}
}

%
%

\markboth{IEEE TRANSACTIONS ON IMAGE PROCESSING, 2019}%
{Shell \MakeLowercase{\textit{et al.}}: Bare Demo of IEEEtran.cls for IEEE Journals}
%



\maketitle

\begin{abstract}
Domain adaptation has been a fundamental technology for transferring knowledge from a source domain to a target domain. The key issue of domain adaptation is how to reduce the distribution discrepancy between two domains in a proper way such that they can be treated indifferently for learning. In this paper, we propose a novel domain adaptation approach, which can thoroughly explore the data distribution structure of target domain.Specifically, we regard the samples within the same cluster in target domain as a whole rather than individuals and assigns pseudo-labels to the target cluster by class centroid matching. Besides, to exploit the manifold structure information of target data more thoroughly, we further introduce a local manifold self-learning strategy into our proposal to adaptively capture the inherent local connectivity of target samples. An efficient iterative optimization algorithm is designed to solve the objective function of our proposal with theoretical convergence guarantee. In addition to unsupervised domain adaptation, we further extend our method to the semi-supervised scenario including both homogeneous and heterogeneous settings in a direct but elegant way.
Extensive experiments on seven benchmark datasets validate the significant superiority of our proposal in both unsupervised and semi-supervised manners.
\end{abstract}

\begin{IEEEkeywords}
domain adaptation, class centroid matching, local manifold self-learning.
\end{IEEEkeywords}

%
\IEEEpeerreviewmaketitle

\section{Introduction \label{Introduction}}
%
%
%
%

\IEEEPARstart{I}{n} many real-world applications, data are generally collected under different conditions, thus hardly satisfying the identical probability distribution hypothesis which is known as a foundation of statistical learning theory. This situation naturally leads to a crucial issue that a classifier trained on a well-annotated source domain cannot be applied to a related but different target domain directly. To surmount this issue, as an important branch of transfer learning, considerable efforts have been devoted to domain adaptation \cite{Pan2009}. By far, domain adaptation has been a fundamental technology for cross-domain
\begin{figure}[htbp]
	\setlength{\abovecaptionskip}{0pt}
	\setlength{\belowcaptionskip}{0pt}
	\renewcommand{\figurename}{Figure}
	\centering
	\includegraphics[width=0.45\textwidth]{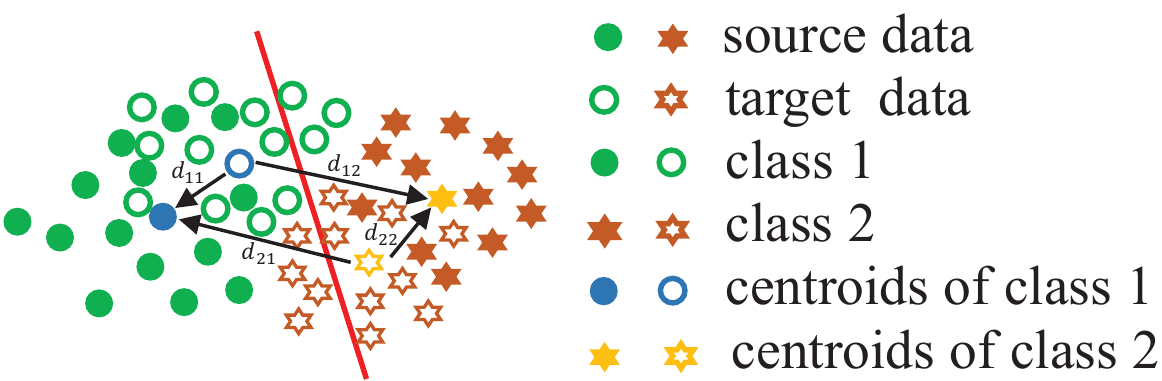}
	\caption{A toy example of misclassification caused by the ignorance of structure information of data distribution, where $d_{ij}$ indicates the distance between the centroids of $i$-th class in target domain and the centroids of $j$-th class in source domain.}
	\label{motivation}
\end{figure}
knowledge discovery, and been considered in various tasks, such as 
object recognition \cite{Guo2013,Rozantsev2018}, face recognition \cite{Ren2014,Qiu2015} and person re-identification \cite{Ma2015}.

\begin{figure*}[h]
	\setlength{\abovecaptionskip}{0pt}
	\setlength{\belowcaptionskip}{0pt}
	\renewcommand{\figurename}{Figure}
	\centering
	\includegraphics[width=0.98\textwidth]{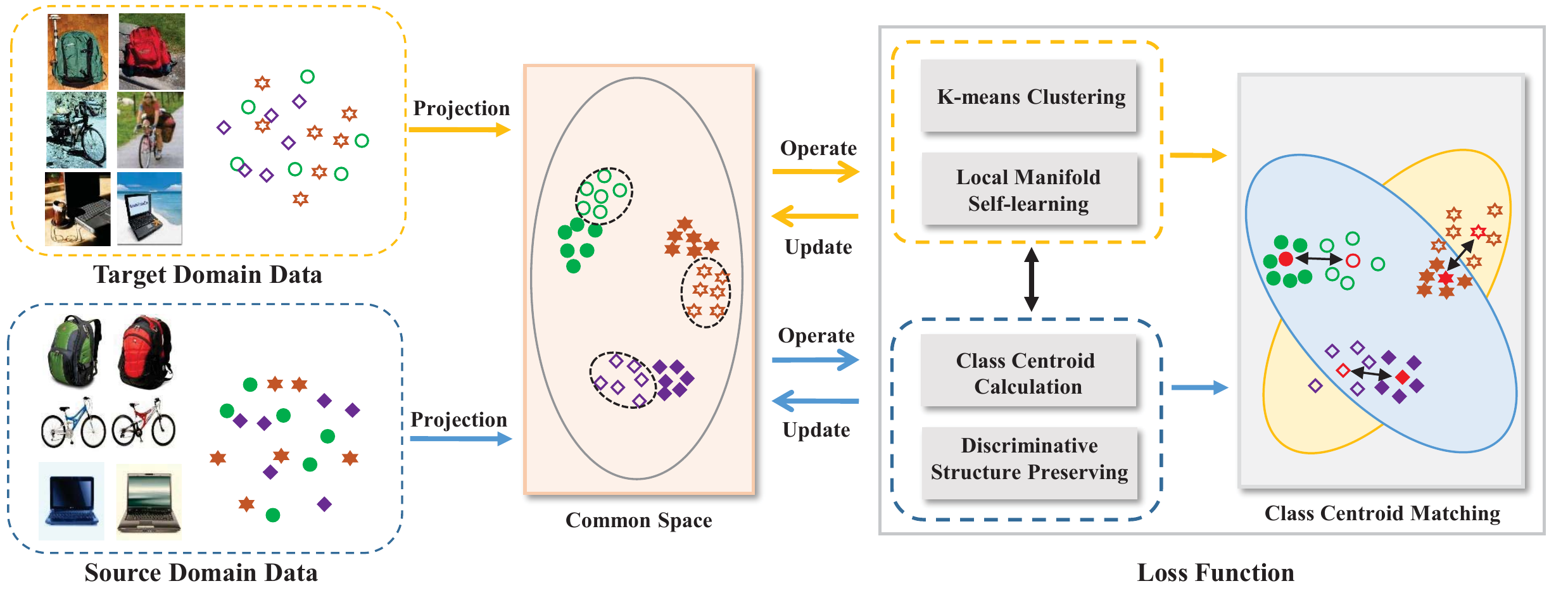}
	\caption{Flowchart of our proposed method. We reformulate the domain adaptation problem as the class centroid matching problem, such that the structure information of data distribution can be exploited. For target data, we further introduce the local manifold self-learning strategy  to explore the inherent local connectivity structure. The yellow and blue pipelines indicate the operations for target domain and source domain, respectively. }
	\label{Framework}
\end{figure*}

The major issue for domain adaptation is how to reduce the distributions difference between the source and target domains \cite{Long2013-ARTL}. Most of recent works aim to seek a common feature space where the distribution difference across domains are minimized \cite{Pan2010,Long2013-JDA,Gong2012,Fernando2013,Sun2016}. To accomplish this, various metrics have been proposed to measure the discrepancy of data distribution, such as correlation distances \cite{COR} or maximum mean discrepancy (MMD) \cite{Gretton2007}.  In this work, we try to thoroughly explore the data distribution structure of target domain with a class centroid matching scheme. To illustrate our motivation more explicitly, a toy example is shown in Fig. \ref{motivation}. The red line is a discriminant hyperplane trained on source data in the projected feature space. As we can see, when the distribution discrepancy between two domains still exists, the hyperplane tends to misclassify target samples if they are labeled independently. In such case, the misclassified samples will seriously mislead the learning of the common feature space in the subsequent iterations, and ultimately cause significant performance drop.  Nevertheless, by the in-depth analysis of Fig. \ref{motivation}, we could further discover that the class centroids in target domain can be well matched to their corresponding class centroids in source domain. In light of this,  when we treat the samples within the same cluster in target domain as a whole  rather than individuals, they could share identical label and obtain accurate classification result.
Motivated by this insight, in this paper, instead of labeling target samples individually, we aim to introduce a novel domain adaptation approach that assigns pseudo-labels to target samples with the guidance of class centroids in two domains, such that the data distribution structure of target domain can be emphasized.


To achieve this goal, the first issue to be handled is the calculation of class centroids in two domains. Due to the labels available in source domain, we can conveniently  acquire the class centroids of source data in the common feature space by calculating the mean of samples in the same class.  For target domain where the labels are absent, we resort to the classical $K$-means clustering algorithm \cite{Macqueen1967} which has been widely used to partition unlabeled data into several groups where similar samples in the same group can be represented by a specific cluster prototype. Intuitively, the cluster prototypes obtained by $K$-means algorithm can be regarded as a good approximation for the class centroids of target domain. After obtaining the class centroids of both source data and target data, the distribution discrepancy minimization problem in domain adaptation can be reformulated as the class centroid matching problem which can be  solved efficiently by the nearest neighbor search.  

The next key concern is how to improve the quality of class centroids, which is crucial to the performance of our approach. In this paper, for labeled source data, we adopt an intuitive strategy that encourages the samples in the identical class to be as close as possible in the projected space, such that the discriminative structure information of source domain can be preserved \cite{Li2018}. As for unlabeled target data, it has been shown that the clustering performance can be significantly enhanced if the local manifold structure is exploited \cite{Goh2007,Goh2008}. Nevertheless, most of existing manifold learning methods highly depend on the predefined similarity matrix built in the original feature space \cite{Saul2003,Belkin2002}, and thus may fail to capture the inherent local structure of high-dimensional data due to the curse of dimensionality. To tackle this problem, inspired by the recently proposed adaptive neighbors learning method \cite{Nie2014}, we introduce a local structure self-learning strategy into our proposal. Specifically, we learn the data similarity matrix according to the local connectivity in the projected low-dimensional feature space rather than the original high-dimensional space, such that the intrinsic local manifold structure of target data can be captured adaptively.

Based on above analysis, a novel domain adaptation method, which can adequately exploit the data distribution structure of target domain by jointly class Centroid Matching and local Manifold Self-learning (CMMS), is naturally proposed.  It is noteworthy that,
more recently, the need for tackling semi-supervised domain adaptation (SDA) problem is growing as there may be some labeled target samples in practice \cite{Hoffman2014,Herath2017,Wang2019,Li-2013,Hsieh2016}. While unsupervised domain adaptation (UDA) methods are well established, most of them cannot be naturally applied to the semi-supervised scenario. Excitingly, the proposed CMMS can be extended to SDA including both homogeneous and heterogeneous settings in a direct but elegant way. The flowchart of our proposed CMMS
is shown in Fig. \ref{Framework}. 
The main contributions of this paper are summarized as follows:
\begin{itemize}
\item We propose a novel domain adaptation method called CMMS, which can thoroughly explore the data distribution structure of target domain via jointly class centroid matching and local manifold self-learning.
\item We present an efficient optimization algorithm to solve the objective function of the proposal, with theoretical convergence guarantee.
\item In addition to unsupervised domain adaptation, we further extend our approach to the semi-supervised scenario including both homogeneous and heterogeneous settings.
\item We conduct extensive evaluation of our method on seven  benchmark datasets, which validates the superior performance of our method in both unsupervised and semi-supervised manners.
\end{itemize}

The rest of this paper is organized as follows. Section \ref{Related work} previews some related literature. Section \ref{Proposed Method} shows our proposal, the optimization algorithm, the convergence and complexity analysis.  We describe our semi-supervised extension in Section  \ref{SDA extension}. Massive experimental results are shown in Section \ref{Experiments}. Finally, we conclude this paper in Section \ref{Conclusions}.

\section{Related Work \label{Related work}}
In this section, we review some previous works closely related to this paper. First, we briefly review the unsupervised domain adaptation methods. Next, related studies of semi-supervised domain adaptation are reviewed. Finally, we introduce some local manifold learning techniques.

\subsection{Unsupervised Domain Adaptation}
Unsupervised domain adaptation aims to handle the scenario where labeled samples are only available from the source domain and there exists different distributions between source and target domains. In the past decades, numerous methods have been proposed to overcome the distribution discrepancy. 

Existing UDA methods can be classified as: 1) instance reweighting \cite{Long2014-TJM, Chen2016}, 2) classifier adaptation \cite{Long2013-ARTL, Wang2018-MEDA}, and 3) feature adaptation \cite{Pan2010,Long2013-JDA,Gong2012,Fernando2013}. We refer the interested readers to \cite{LeiZhang2019}, which contains an excellent survey. Our proposal falls into the third category, i.e., feature adaptation, which addresses domain shift by either searching intermediate subspaces to achieve domain transfer \cite{Gong2012,Fernando2013} or learning a common feature space where the source and target domains have similar distributions \cite{Pan2010,Long2013-JDA}. In this paper, we focus on the latter line. Among existing works, TCA \cite{Pan2010} is a pioneering approach, which learns a transformation matrix to align marginal distribution between two domains via MMD. Later, JDA \cite{Long2013-JDA} considers conditional distribution alignment by forcing the class means to be close to each other. In the subsequent research, several works further propose to employ the discriminative information to facilitate classification performance. For instance, Li \emph{et al}. \cite{Li2018} utilized the discriminative information for the source and target domains by encouraging intra-class compactness and inter-class dispersion. Liang \emph{et al}. \cite{Liang2018} achieved this goal by promoting class clustering.

Despite the promising performance, all the above methods classify target samples independently, which may cause misclassification since the structure information of data distribution in target domain is ignored. To tackle this issue, several recent works attempt to exploit the distribution structure of target data via clustering. For example, Liang \emph{et al}. \cite{Liang2019} proposed to seek a subspace where the target centroids are forced to approach those in the source domain. Inspired by the fact that the target samples are well clustered in the deep feature space, Wang \emph{et al}. \cite{Wang2019-SPL} proposed a selective pseudo-labeling approach (SPL) based on structured prediction, which is the most relevant to our proposal. Nevertheless, our proposal is significantly different from it. First, these two approaches are essentially two different formalization frameworks. To be specific, our proposal integrates the projection matrix learning, the class centroid calculation and the class centroid matching into a unified optimization objective, which is more intuitive and concise. Second, in our framework, we further present the specific strategies for source  and target data respectively, including discriminative structure preserving for source data and local manifold self-learning for target data, to improve the quality of class centroids. Third, our formalization owns better scalability, which can be naturally extended to semi-supervised scenario with little  modification of optimization algorithm.

More recently, deep domain adaptation methods has attracted much more attention in the past few years  \cite{LeiZhang2019}, which are generally achieved by adding adaptation layers to align the means of distributions \cite{DAN,RTN,TPN} or adding a subnetwork to confuse the source and target domains \cite{UDABP, ADDA, IADDA, SimNet,PFAN}. For instance, Pinheiro \cite{SimNet} compared the embedding of a target image with the source prototypes and the label that best matches it is given. Chen \emph{et al}. \cite{PFAN} employed a different manner by matching the source and target prototypes for each class via adaptive prototype alignment. Pan \emph{et al}. \cite{TPN} innovatively proposed transferrable prototypical networks (TPN) for unsupervised domain adaptation, which enforces the prototypes of each class in source and target domains to be close in the embedding space and the score distributions predicted by prototypes separately on source and target data to be similar. Both TPN and our CMMS share the similar idea of using prototypes to represent the class distribution. Nevertheless, there are significant differences between the two approaches. First, in the process of assigning pseudo-label to target data, TPN treats the target samples as individuals and independently matches each target sample to the nearest prototype in the source domain. As shown in the toy example in Fig.\ref{motivation}, the drawback of such procedure lies in the neglect of  the distribution structure of target data. We handle this issue by first treating the samples within the same cluster in target domain as a whole, and then assigning the label of nearest source class to the target cluster. Second, the exploration of local manifold structure of target data plays a vital role in the learning of common space, which is ignored by TPN. By contrast, we introduce local manifold self-learning strategy into our CMMS to capture the intrinsic local connectivity.

\subsection{Semi-supervised Domain Adaptation}
Unlike the unsupervised domain adaptation that no labels are available, in practice, a more common scenario is that the target domain contains a few labeled samples. Such scenario leads to a promising research direction, which is referred as the semi-supervised domain adaptation. 

According to the property of sample features, SDA algorithms are developed in two different settings: 1) homogeneous setting, i.e., the source and target data are sampled from the same feature space; 2) heterogeneous setting, i.e., the source and target data often have different feature dimensions. In the homogeneous setting, the labeled target samples are used in various ways. For example, Hoffman \emph{et al.} \cite{Hoffman2014} jointly learnt the transformation matrix and classifier parameters, forcing the source and target samples with identical label have high similarity. Similarly, Herath \emph{et al.} \cite{Herath2017} proposed to learn the structure of a Hilbert space to reduce the dissimilarity between labeled samples and further match the source and target domains via the second order statistics. Recently, based on Fredholm integral, Wang \emph{et al.} \cite{Wang2019} proposed to learn a cross-domain kernel classifier that can classify the labeled target data correctly using square loss function or hinge loss function. In the heterogeneous setting, relieving feature discrepancy and reducing distribution divergence are two inevitable issues \cite{Li2018-PA}. For the first issue, one incredibly simple approach \cite{Li-2013} is to use the original features or zeros to augment each transformed sample into same size. Another natural approach \cite{Hsieh2016} is to learn two projection matrices to derive a domain-invariant feature subspace, one for each domain. Recently, after employing two matrices to project the source and target data to a common feature space, Li \emph{et al.} \cite{Li2018-PA} employed a shared codebook to match the new feature representations on the same bases. For the second issue, one favorite solution is to minimize the MMD distance of the source and target domains \cite{Hsieh2016,Li2018-PA}. Additionally, Tsai \emph{et al.} \cite{Tsai2016} proposed a representative landmark selection approach, which is similar to instance reweighting in the UDA scenario. When we obtain a limited amount of labeled target samples,  manifold regularization, an effective strategy for semi-supervised learning, has also been employed by several previous works \cite{Xiao2014,Yao2015}.

In contrast to these SDA methods, our semi-supervised extension is quite simple and intuitive. To be specific, the labeled target data are used to improve the cluster prototypes learning of the unlabeled target data. Besides, connections between the labeled and unlabeled target data are built, which is a common strategy to develop a semi-supervised model. Different from the homogeneous setting where a unified projection is learnt, we learn two projection matrices in the heterogeneous setting like \cite{Hsieh2016}. Notably, the resulting optimization problems in two settings own the same standard formula and can be solved by the same algorithm in UDA scenario with just very tiny modifications.

\subsection{Local Manifold Learning \label{SLG}}
The goal of local manifold learning is to capture the underlying local manifold structure of the given data in the original high-dimensional space and preserve it in the low-dimensional embedding. Generally, local manifold learning methods contain three main steps: 1) selecting neighbors; 2) computing affinity matrix; 3) calculating the low-dimensional embedding \cite{Hong2017}. 

Local linear embedding \cite{Saul2003} and Laplacian eigenmaps \cite{Belkin2002} are two typical methods. In local linear embedding, the local manifold structure is captured by linearly reconstructing each sample using the corresponding neighbors in the original space and the reconstruction coefficients are preserved in the low-dimensional space. In Laplacian eigenmaps, the adjacency matrix of given data is obtained in the original feature space using Gaussian function. However, the local manifold structure is artificially captured using pairwise distances with heat kernel, which brings relatively weak representation for the ignorance of the properties of local neighbors \cite{Hong2017}. Recently, to learn a more reliable adjacency matrix, Nie \emph{et al.} \cite{Nie2014} proposed to assign the neighbors of each sample adaptively based on the Euclidean distances in the low-dimensional space. This strategy has been widely utilized in clustering \cite{Zhan2018}, feature selection \cite{Hou2017} and feature representation learning \cite{Wang2018-FML}. 

In domain adaptation problems, several works have borrowed the advantages of local manifold learning. For example, Long \emph{et al.} \cite{Long2013-ARTL} and Wang \emph{et al.} \cite{Wang2018-MEDA} employed manifold regularization to maintain the manifold consistency underlying the marginal distributions of two domains. Hou \emph{et al.} \cite{Hou2016} and Li \emph{et al.} \cite{Li2019-LPJT} used label propagation to predict target labels. However, they all calculate adjacency matrix in the original high-dimensional space with the predefined distance measurement, which is unreliable due to the curse of dimensionality. By contrast, our proposal can capture and employ the inherent local manifold  structure of target data adaptively, thus leading to superior performance.

\section{Proposed Method \label{Proposed Method}}
In this section,  we first introduce the notations and basic concepts used throughout this paper. Then, the details of our approach are described. Next, an efficient algorithm is designed to solve the optimization problem of our proposal. Finally, the convergence and complexity analysis of the optimization algorithm are given.

\subsection{Notations}
A domain $\mathcal{D}$ contains a feature space $\chi$ and a marginal probability distribution $P(\mathbf{X})$, where $\mathbf{X} \in \chi$. For a specific domain, a task $\mathcal{T}$ consists of a label space  $\mathcal{Y}$ and a labeling function $f(\mathbf{x})$, denoted by $\mathcal{T} = \{\mathcal{Y}, f(\mathbf{x})\}$ \cite{Pan2009}. For simplicity, we use subscripts $s$ and $t$ to describe the source domain and target domain, respectively.

We denote the source domain data as $\mathcal{D}_s = \{\mathbf{X}_s,\mathbf{Y}_s\} = \{(\mathbf{x}_{si},y_{si})\}_{i=1}^{n_s}$, where $\mathbf{x}_{si} \in \mathbb{R}^{m}$ is a source sample and $y_{si} \in \mathbb{R}$ is the corresponding label. Similarly, we denote the target domain data as $\mathcal{D}_t = \{\mathbf{X}_t\} = \{\mathbf{x}_{tj}\}_{j=1}^{n_t}$, where $\mathbf{x}_{tj} \in \mathbb{R}^{m}$. For clarity, we show the key notations used in this paper and the corresponding descriptions in Table \ref{notations}.
\begin{table}[]
\centering
\caption{Frequently used notations and descriptions}
\label{notations}
\begin{tabular}{ccc}
\toprule
Notation & \quad Description\\
\midrule
$\mathbf{X}_s/\mathbf{X}_t$&  \quad source/target original data\\
$n_s/n_t$&  \quad number of source/target samples\\
$\mathbf{P}$& \quad projection matrix\\
$\mathbf{F}$&  \quad target cluster centroids\\
$\mathbf{G}_t$&  \quad target label matrix\\
$\mathbf{S}$&  \quad target adjacency matrix\\
$\mathbf{H}$& \quad  centering matrix\\
$\mathbf{I}_d$ & \quad identity matrix with dimension $d$\\
$m$ & \quad dimension of original features\\
$d$ & \quad dimension of projected features\\
$C$&  \quad number of shared class\\
$n_s^c$&  \quad number of source samples in class $c$\\
$\mathbf{0}_{p \times q} / \mathbf{1}_{p \times q}$ &  \qquad a matrix of size $p \times q$ with all elements as $0 / 1$ \\
$\mathbf{0}_{p}/\mathbf{1}_{p}$ & \quad a column vector of size $p$ with all elements as $0 / 1$ \\
\bottomrule
\end{tabular}
\end{table}
\subsection{Problem Formulation \label{Problem Formulation}}
The core idea of our CMMS lies in the emphasis on data distribution structure by class centroid matching of two domains and local manifold structure self-learning for target data. The overall framework of CMMS can be stated by the following formula:
\begin{equation}
\label{E1}
\min_{\mathbf{P},\mathbf{F},\mathbf{G}_t,\mathbf{S}}\Omega(\mathbf{P},\mathbf{F}) + \alpha \Theta(\mathbf{P},\mathbf{F},\mathbf{G}_t) + \gamma \Psi(\mathbf{P},\mathbf{S})+ \beta \Phi(\mathbf{P})
\end{equation}
The first term $\Omega(\mathbf{P},\mathbf{F})$ is used to match class centroids. $\Theta(\mathbf{P},\mathbf{F},\mathbf{G}_t)$ is the clustering term for target data in the projected space. $\Psi(\mathbf{P},\mathbf{S})$ is employed to capture the data structure information. $\Phi(\mathbf{P})$ is the regularization term to avoid overfitting. Hyper-parameters $\alpha$, $\beta$ and $\gamma$ are employed to balance the influence of different terms. Next, we will introduce these items in detail.

\subsubsection{$K$-means Clustering for Target Data} In our CMMS, we borrow the idea of clustering to obtain the cluster prototypes which can be regarded as the pseudo class centroids. In such case, the sample distribution structure information of target data can be acquired. To achieve this goal, various existing clustering algorithms can be our candidates. Without loss of generality, for the sake of simplicity, we adopt the classical $K$-means algorithm to get the cluster prototypes in this paper. Thus, we have the following formula:
\begin{equation}
\label{E2}
\Theta(\mathbf{P},\mathbf{F},\mathbf{G}_t)=\|\mathbf{P}^{\mathrm{T}}\mathbf{X}_t-\mathbf{FG}_t^\mathrm{T}\|_F^2\\
\end{equation}
where $\mathbf{P} \in \mathbb{R}^{m \times d}$ is the projection matrix, $\mathbf{F} \in \mathbb{R}^{d \times C}$ is the cluster centroids of target data, $\mathbf{G}_t \in \mathbb{R}^{n_t \times C}$ is the cluster indicator matrix of target data which is defined as $(\mathbf{G}_t)_{ij} = 1$ if the cluster label of $\mathbf{x}_{ti}$ is $j$, and $(\mathbf{G}_t)_{ij} = 0$ otherwise.

\subsubsection{Class Centroid Calculation for Source Data and Class Centroid Matching of Two Domains} 
Once the cluster prototypes of target data are obtained, we can reformulate the distribution discrepancy minimization problem in domain adaptation as the class centroid matching problem.  Note that the class centroids of source data can be obtained exactly by calculating the mean value of sample features in the identical class. 
In this paper, we solve the class centroid matching problem by the nearest neighbor search since it is simple and efficient. Specifically, we search the nearest source class centroid for each target cluster centroid, and minimize the sum of distance of each pair of class centroids. Finally, the class centroid matching of two domains is formulated as: 
\begin{equation}
\label{E3}
\Omega(\mathbf{P},\mathbf{F})=\|\mathbf{P}^\mathrm{T}\mathbf{X}_s\mathbf{E}_s-\mathbf{F}\|_F^2
\end{equation}
where $\mathbf{E}_s \in \mathbb{R}^{n_s \times C}$ is a constant matrix used to calculate the class centroids of source data in the projected space with each element $\mathbf{E}_{ij} = 1/n_s^j$ if $y_{si} = j$, and $\mathbf{E}_{ij} = 0$ otherwise.

\subsubsection{Local Manifold Self-Learning for Target Data}
In our proposed CMMS, the cluster prototypes of target samples are actually the approximation of their corresponding class centroids. Hence, the quality of cluster prototypes plays an important role in the final performance of our CMMS. Existing works have proven that the performance of clustering can be significantly improved by the exploiting of local manifold structure. Nevertheless, most of them highly depend on the predefined adjacent matrix in the original feature space, and thus fail to capture the inherent local manifold structure of high-dimensional data due to the curse of dimensionality. For this issue,  inspired by the recent work \cite{Nie2014}, we propose to introduce a local manifold self-learning strategy into our CMMS. Instead of predefining the adjacent matrix in the original high-dimensional space, we adaptively learn the data similarity according to the local connectivity in the projected low-dimensional space, such that the intrinsic local manifold structure of target data can be captured. 
The formula of local manifold self-learning is shown as follows:
\begin{equation}
\label{E4}
\begin{split}
&\Psi_t(\mathbf{P},\mathbf{S}) = \sum\limits_{i,j=1}^{n_t}\|\mathbf{P}^\mathrm{T}\mathbf{x}_{ti}-\mathbf{P}^\mathrm{T}\mathbf{x}_{tj}\|_2^2S_{ij}+\delta{S_{ij}^2}\\
& \qquad \quad \ \ = 2\mathrm{tr}(\mathbf{P}^\mathrm{T}\mathbf{X}_t\mathbf{L}_t\mathbf{X}_t^\mathrm{T}\mathbf{P}) + \delta\|\mathbf{S}\|_F^2\\
&\qquad \quad s.t.\  \ \mathbf{S}\mathbf{1}_{n_t} = \mathbf{1}_{n_t}, 0 \leq S_{ij} \leq 1
\end{split}
\end{equation}
where $\mathbf{S} \in \mathbb{R}^{n_t \times n_t}$ is the adjacency matrix in target domain and $\delta$ is a hyper-parameter. $\mathbf{L}_t$ is the corresponding graph laplacian matrix calculated by $\mathbf{L}_t = \mathbf{D} - \mathbf{S}$, where $\mathbf{D}$ is a diagonal matrix with each element $\mathbf{D}_{ii} = \sum_{j \neq i}\mathbf{S}_{ij}$.
\subsubsection{Discriminative Structure Preserving for Source Data}The above descriptions have highlighted the main components of our CMMS. Intuitively, a reasonable hypothesis for source data is that the samples in the identical class should be as close as possible in the projected space, such that the discriminative structure information of source domain can be preserved. As one trivial but effective trick, inspired by \cite{Li2018}, we formulate this thought as follows:
\begin{equation}
\label{E5}
\begin{aligned}
\Psi_s(\mathbf{P}) &= \sum\limits_{c=1}^C \frac {1}{n_s^c}\sum\limits_{y_{si},y_{sj} = c}\|\mathbf{P}^\mathrm{T}\mathbf{x}_{si}-\mathbf{P}^\mathrm{T}\mathbf{x}_{sj}\|_2^2 \\
=\sum\limits_{i,j=1}^{n_s}&\|\mathbf{P}^\mathrm{T}\mathbf{x}_{si}-\mathbf{P}^\mathrm{T}\mathbf{x}_{sj}\|_2^2W_{ij} = 2\mathrm{tr}(\mathbf{P}^\mathrm{T}\mathbf{X}_s\mathbf{L}_s\mathbf{X}_s^\mathrm{T}\mathbf{P})
\end{aligned}
\end{equation}
where $\mathrm{tr}(\cdot)$ is the trace operator and $\mathbf{W}$ is the similarity matrix for source data, which is defined as $W_{ij} = 1/n_s^c$ if $y_{si}=y_{sj} = c$, and 0, otherwise. The coefficient $1/n_s^c$ is used to remove the effects of different class sizes \cite{Wang2016}. $\mathbf{L}_s$ is the respective laplacian matrix and is defined as:
\begin{equation*}
(\mathbf{L}_s)_{ij} =
\begin{cases}
1 - \frac{1}{n_s^c}, &\text{if $i = j$;}\\
-\frac {1}{n_s^c}, &\text{if $ i \neq j, y_{si}=y_{sj} = c$;}\\
0, &\text{otherwise.}
\end{cases}
\end{equation*}

For simplicity, we denote $\mathbf{X} = [\mathbf{X}_s,\mathbf{X}_t]$ and $\mathbf{L} = 
\mathrm{diag}(2\mathbf{L}_s,2\mathbf{L}_t)$.  By combining Eq.(\ref{E4}) and Eq.(\ref{E5}), we obtain a general $\Psi(\mathbf{P},\mathbf{S})$ term which can capture the diverse structure information of both source and target data:
\begin{equation}
\label{E6}
\begin{split}
\Psi(\mathbf{P},\mathbf{S})
&=2\mathrm{tr}(\mathbf{P}^\mathrm{T}\mathbf{X}_s\mathbf{L}_s\mathbf{X}_s^\mathrm{T}\mathbf{P}) + 2\mathrm{tr}(\mathbf{P}^\mathrm{T}\mathbf{X}_t\mathbf{L}_t\mathbf{X}_t^\mathrm{T}\mathbf{P})+\delta\|\mathbf{S}\|_F^2 \\
=\mathrm{tr}(\mathbf{P}&^\mathrm{T}[\mathbf{X}_s, \mathbf{X}_t]\begin{bmatrix} 2\mathbf{L}_s & \mathbf{0}_{n_s \times n_t} \\ \mathbf{0}_{n_t \times n_s} & 2\mathbf{L}_t \end{bmatrix}[\mathbf{X}_s, \mathbf{X}_t]^\mathrm{T}\mathbf{P}) +\delta\|\mathbf{S}\|_F^2 \\
=\mathrm{tr}(\mathbf{P}&^\mathrm{T}\mathbf{X}\mathbf{L}\mathbf{X}^\mathrm{T}\mathbf{P}) + \delta\|\mathbf{S}\|_F^2\\
s.t.\  \ \mathbf{S}&\mathbf{1}_{n_t} = \mathbf{1}_{n_t}, \ \ 0 \leq S_{ij} \leq 1
\end{split}
\end{equation}

Besides, to avoid overfitting and improve the generalization capacity, we further add an $F$-norm regularization term to the projection matrix $\mathbf{P}$:
\begin{equation}
\label{E7}
\Phi(\mathbf{P}) = \|\mathbf{P}\|_F^2
\end{equation}

So far, by combining Eq.(\ref{E2}), (\ref{E3}), (\ref{E6}) and (\ref{E7}), we arrive at our final CMMS formulation:
\begin{equation}
\label{E8}
\begin{split}
&\min_{\mathbf{P}, \mathbf{F}, \mathbf{G}_t, \mathbf{S}}\|\mathbf{P}^{\mathrm{T}}\mathbf{X}_s\mathbf{E}_s-\mathbf{F}\|_F^2 + \alpha\|\mathbf{P}^{\mathrm{T}}\mathbf{X}_t-\mathbf{F}\mathbf{G}_t^{\mathrm{T}}\|_F^2 \\
&\qquad\quad+\beta \|\mathbf{P}\|_F^2+\gamma (\mathrm{tr}(\mathbf{P}^\mathrm{T}\mathbf{X}\mathbf{L}\mathbf{X}^{\mathrm{T}}\mathbf{P})+\delta\|\mathbf{S}\|_F^2)\\
&\qquad s.t.\ \mathbf{P}^{\mathrm{T}}\mathbf{XHX}^{\mathrm{T}}\mathbf{P} = \mathbf{I}_d,\ \mathbf{G}_t {\in \{0,1\}^{n_{t} \times C}}, \\
& \qquad \qquad \ \mathbf{S}\mathbf{1}_{n_t} = \mathbf{1}_{n_t}, \ \ 0 \leq S_{ij} \leq 1
\end{split}
\end{equation}
where $\mathbf{I}_d$ is an identity matrix of dimension $d$ and $\mathbf{H}$ is centering matrix defined as $\mathbf{H} = \mathbf{I}_{n_s+n_t}- \frac{1}{n_s+n_t}\mathbf{1}_{(n_s+n_t) \times (n_s+n_t)}$. The first constraint in (\ref{E8}) is inspired by principal component analysis, which aims to maximize projected data variance \cite{Long2013-JDA}. For the sake of simplified format, we reformulate the objective function in (\ref{E8}) as the following standard formula:
\begin{equation}
\label{E9}
\begin{split}
&\min_{\mathbf{P}, \mathbf{F}, \mathbf{G}, \mathbf{S}}\|\mathbf{P}^\mathrm{T}\mathbf{X}\mathbf{E}-\mathbf{F}\|_F^2 + \alpha\|\mathbf{P}^\mathrm{T}\mathbf{X}\mathbf{V}-\mathbf{F}\mathbf{G}^\mathrm{T}\|_F^2 \\
&\qquad \ \ +\beta \|\mathbf{P}\|_F^2 +\gamma (\mathrm{tr}(\mathbf{P}^\mathrm{T}\mathbf{X}\mathbf{L}\mathbf{X}^\mathrm{T}\mathbf{P})+\delta\|\mathbf{S}\|_F^2) \\
&\qquad s.t.\ \mathbf{P}^\mathrm{T}\mathbf{XHX}^\mathrm{T}\mathbf{P} = \mathbf{I}_d,\ \mathbf{G} {\in \{0,1\}^{n \times C}}, \\
&\qquad\qquad  \ \mathbf{S}\mathbf{1}_{n_t} = \mathbf{1}_{n_t}, \ \ 0 \leq S_{ij} \leq 1
\end{split}
\end{equation}
where $n = n_s + n_t$, $\mathbf{V} = \mathrm{diag}(\mathbf{0}_{n_s \times n_s}, \mathbf{I}_{n_t})$, $\mathbf{E} = [\mathbf{E}_s; \mathbf{0}_{n_t \times C}]$, $\mathbf{G} = [\mathbf{0}_{n_s \times C};\mathbf{G}_t]$ and $\mathbf{L} = \mathrm{diag}(\mathbf{L}_s,\mathbf{L}_t)$. 

\subsection{Optimization Procedure}
According to the objective function of our CMMS in Eq.(\ref{E9}), 
there are four variables  $\mathbf{F}$, $\mathbf{P}$, $\mathbf{G}$, $\mathbf{S}$ that need to be optimized. Since it is not jointly convex for all variables, we update each of them alternatively while keeping the other variables fixed. Specifically, each subproblem is solved as follows:

\textbf{1. $\mathbf{F}$-subproblem:}  
When $\mathbf{P}$, $\mathbf{G}$ and $\mathbf{S}$ are fixed, the optimization problem (\ref{E9}) becomes:
\begin{equation}
\label{E10}
\min_{\mathbf{F}}\|\mathbf{P}^\mathrm{T}\mathbf{X}\mathbf{E}-\mathbf{F}\|_F^2 + \alpha\|\mathbf{P}^\mathrm{T}\mathbf{X}\mathbf{V}-\mathbf{F}\mathbf{G}^\mathrm{T}\|_F^2
\end{equation}
By setting the derivative of (\ref{E10}) with respect to $\mathbf{F}$ as 0, we obtain:
\begin{equation}
\label{E11}
\mathbf{F} = (\mathbf{P}^\mathrm{T}\mathbf{X}\mathbf{E} + \alpha \mathbf{P}^\mathrm{T}\mathbf{XVG})(\alpha \mathbf{G}^\mathrm{T}\mathbf{G} + \mathbf{I}_C)^{-1}
\end{equation}

\textbf{2. $\mathbf{P}$-subproblem:} 
Substituting Eq.(\ref{E11}) into Eq.(\ref{E9}) to replace $\mathbf{F}$, we can get the following subproblem:
\begin{equation}
\label{E12}
\begin{split}
&\min_{\mathbf{P}} \ \mathrm{tr}(\mathbf{P}^\mathrm{T}(\mathbf{XRX}^\mathrm{T} + \gamma \mathbf{XLX}^\mathrm{T} + \beta \mathbf{I}_m)\mathbf{P})\\
&\qquad s.t.\ \mathbf{P}^\mathrm{T}\mathbf{XHX}^\mathrm{T}\mathbf{P} = \mathbf{I}_d
\end{split}
\end{equation}
where $\mathbf{R} = \mathbf{EE}^\mathrm{T}-\mathbf{E}(\alpha \mathbf{G}^\mathrm{T}\mathbf{G} + \mathbf{I}_C)^{-1}(\mathbf{E}+ \alpha \mathbf{VG})^\mathrm{T} + \alpha \mathbf{VV}^\mathrm{T}- \alpha \mathbf{VG}(\alpha \mathbf{GG}^\mathrm{T} + \mathbf{I}_C)^{-1}(\mathbf{E}+ \alpha \mathbf{VG})^\mathrm{T}$. The above problem can be transformed to a generalized eigenvalue problem as follows:
\begin{equation}
\label{E13}
(\mathbf{XRX}^\mathrm{T} + \gamma \mathbf{XLX}^\mathrm{T} + \beta \mathbf{I}_m)\mathbf{P} = \mathbf{XHX}^\mathrm{T}\mathbf{P} \mathbf{\Pi}
\end{equation}
where $\mathbf{\Pi} = \mathrm{diag}(\pi_1,\pi_2,...,\pi_d) \in \mathbb{R}^{d \times d}$ is a diagonal matrix with each element as a Lagrange Multiplier. Then the optimal solution is obtained by calculating the eigenvectors of Eq.(\ref{E13}) corresponding to the $d$-smallest eigenvalues.

\begin{algorithm}[h]
	\SetAlgoLined
	\caption{CMMS for UDA}
	\label{alg1}
	\KwIn{Source data $\{\mathbf{X}_s,\mathbf{Y}_s\}$;
		target data $\{\mathbf{X}_t\}$;
		initial target label matrix $\mathbf{G}_t$;
		initial  adjacency matrix $\mathbf{S}$;	
		hyper-parameters $\gamma=5.0$, $\alpha$, $\beta$; subspace dimensionality $d$ = 100;
		neighborhood size $k$ = 10; maximum iteration $T$ = 10. }
	\KwOut{Target label matrix $\mathbf{G}_t$}
	\BlankLine
	$t$ = 0;\\
	\While { \rm{not converge} \textbf{and} $t$ $\leq T$}
	{
		// \textit{Projection matrix $\mathbf{P}$} \\
		Update $\mathbf{P}$ by solving the generalized eigenvalue problem in (\ref{E13});\\
		// \textit{Cluster prototype matrix of target data $\mathbf{F}$} \\
		Update $\mathbf{F}$ by (\ref{E11});\\
		// \textit{Label assignment matrix of target data $\mathbf{G}_t$} \\
		Update each row of $\mathbf{G}_t$ by (\ref{E14});\\
		// \textit{Local adjacency matrix of target data $\mathbf{S}$} \\
		Update each row of $\mathbf{S}$ by solving (\ref{E16});\\
		$t$ = $t$ + 1;
	}
	\textbf{Return} Target label matrix $\mathbf{G}_t$.
\end{algorithm}

\textbf{3. $\mathbf{G}$-subproblem:}
In variable $\mathbf{G}$, only $\mathbf{G}_t$ needs to be updated. With $\mathbf{P}$, $\mathbf{F}$ and $\mathbf{S}$ fixed, the optimization problem with regard to $\mathbf{G}_t$ is equal to minimizing Eq.(\ref{E2}). Like $K$-means clustering, we can solve it by assigning the label of each target sample to its nearest cluster centroid. To this end, we have:
\begin{equation}
\label{E14}
  (\mathbf{G}_t)_{ik} =
  \begin{cases}
    1, &\text{if $k = \mathop{\arg\min}_{j}\|\mathbf{P}^\mathrm{T}\mathbf{x}_{ti} -\mathbf{F}(:,j)\|_2^2$}\\
	0, &\text{otherwise}
  \end{cases}
\end{equation}

\textbf{4. $\mathbf{S}$-subproblem:}
When $\mathbf{G}$, $\mathbf{F}$ and $\mathbf{P}$ are fixed, the optimization problem with regard to $\mathbf{S}$ is equal to minimizing Eq.(\ref{E5}). Actually, we can divide it into $n_t$ independent subproblems with each formulated as:
\begin{equation}
\label{E15}
\min_{\mathbf{S}_{i,:}\mathbf{1}_{n_t} = 1, 0 \leq S_{ij} \leq 1}\sum\limits_{j=1}^{n_t}\|\mathbf{P}^\mathrm{T}\mathbf{x}_{ti}-\mathbf{P}^\mathrm{T}\mathbf{x}_{tj}\|_2^2S_{ij}+\delta{S_{ij}^2}
\end{equation}
where $\mathbf{S}_{i,:}$ is the $i$-th row of $\mathbf{S}$. By defining $A_{ij} = \|\mathbf{P}^\mathrm{T}\mathbf{x}_{ti}-\mathbf{P}^\mathrm{T}\mathbf{x}_{tj}\|_2^2$, the above problem can be written as:
\begin{equation}
\label{E16}
\min_{\mathbf{S}_{i,:}\mathbf{1}_{n_t} = 1,0 \leq S_{ij} \leq 1}\|\mathbf{S}_{i,:} + \frac {\mathbf{A}_{i,:}}{2\delta}\|_2^2
\end{equation}
The corresponding Lagrangian function is:
\begin{equation}
\label{E17}
\min_{\mathbf{S}_{i,:}}\|\mathbf{S}_{i,:} + \frac {\mathbf{A}_{i,:}}{2\delta}\|_2^2 - \mu(\mathbf{S}_{i,:}\mathbf{1}_{n_t} - 1) - \mathbf{S}_{i,:}\bm{\eta}^{\mathrm{T}}
\end{equation}
where $\mu$ and $\bm{\eta}$ are the Lagrangian multipliers. 

To explore the data locality and reduce computation time, we prefer to learn a sparse $\mathbf{S}_{i,:}$, i.e., only the $k$-nearest neighbors of each sample are preserved to be locally connected. Based on the KKT condition, Eq.(\ref{E17}) has a closed-form solution:
\begin{equation}
\label{E18}
S_{ij} = \mathrm{max}(z-\frac{A_{ij}}{2\delta},0), \quad z = \frac{1}{k} + \frac{1}{2k\delta}\sum\limits_{j=1}^{k}\tilde{A}_{ij}
\end{equation}
where $\tilde{A}_{ij}$ is the element of matrix $\tilde{\mathbf{A}}$, obtained by sorting the entries for each row of $\mathbf{A}$ in an ascending order. According to \cite{Nie2014}, we define $B_{ij} = \|\mathbf{x}_{ti}-\mathbf{x}_{tj}\|_2^2$ and set the value of parameter $\delta$ as:
\begin{equation}
\label{E19}
\delta =\frac{1}{n_t}\sum\limits_{i = 1}^{n_t}(\frac{k}{2}\tilde{B}_{i,k+1}-\frac{1}{2}\sum\limits_{j=1}^{k}\tilde{B}_{ij})
\end{equation}
Similar to $\tilde{A}_{ij}$, we also define $\tilde{B}_{ij}$ as the element of matrix $\tilde{\mathbf{B}}$ which is obtained by sorting the entries for each row of $\mathbf{B}$ from small to large.

We use a linear SVM\footnote{\url{https://www.csie.ntu.edu.tw/~cjlin/liblinear/}\label{SVM}} classifier to initialize the target label matrix.  The initial adjacency matrix is obtained by solving each subproblem like (\ref{E16}) in the original space. The detailed optimization steps of CMMS are summarized in Algorithm \ref{alg1}.

\subsection{Convergence and Complexity Analysis \label{Algorithm Analysis}}
\subsubsection{Convergence Analysis}
We can prove the convergence of the proposed Algorithm \ref{alg1} via the following proposition:
\begin{proposition}
The proposed iterative optimization steps shown in algorithm \ref{alg1} monotonically decreases the objective function value of (\ref{E9}) in each iteration.
\end{proposition}
\begin{proof}
	Assume that at the $r$-th iteration, we get $\mathbf{P}_r$, $\mathbf{F}_r$, $\mathbf{G}_r$, $\mathbf{S}_r$. We denote the value of the objective function in (\ref{E9}) at the $r$-th iteration as $f(\mathbf{P}_r, \mathbf{F}_r, \mathbf{G}_r, \mathbf{S}_r)$. In our Algorithm \ref{alg1}, we divide problem
	(\ref{E9}) into four subproblems (\ref{E10}), (\ref{E12}), (\ref{E14}) and (\ref{E15}), and each of them is a convex problem with respect to their corresponding variables. By solving the subproblems alternatively, our proposed algorithm can ensure finding the optimal solution of each subproblem, i.e., $\mathbf{P}_{r+1}$, $\mathbf{F}_{r+1}$, $\mathbf{G}_{r+1}$, $\mathbf{S}_{r+1}$. Therefore, as the combination of four subproblems, the objective function value of (\ref{E9}) in the $(r+1)$-th iteration satisfies:
	\begin{equation}
	\label{E26}
	f(\mathbf{P}_r, \mathbf{F}_r, \mathbf{G}_r, \mathbf{S}_r) \ge f(\mathbf{P}_{r+1}, \mathbf{F}_{r+1}, \mathbf{G}_{r+1}, \mathbf{S}_{r+1})
	\end{equation}
	In light of this, the proof is completed and the algorithm will converge to local solution at least.
\end{proof}

\subsubsection{Complexity Analysis}
The optimization Algorithm \ref{alg1} of our CMMS comprises four subproblems. The complexity of these four subproblems are induced as follows: First, the cost of initializing $\mathbf{S}$ is $\mathcal{O}(mn_t^2 + n_t^2\mathrm{log}(n_t))$ and we ignore the time to initialize $\mathbf{G}_t$ since the base classifier is very fast. Then, in each iteration, the complexity to construct and solve the generalized eigenvalue problem (\ref{E12}) for $\mathbf{P}$ is $\mathcal{O}(n^2C + nC^2 + C^3 + mn^2 + nm^2+ dm^2)$. The target cluster centroids $\mathbf{F}$ can be obtained with a time cost of $\mathcal{O}(dmn + dnC + dC^2)$. The complexity of updating the target labels matrix $\mathbf{G}_t$ is $\mathcal{O}(Cdn_t)$. The adjacency matrix $\mathbf{S}$ is updated with the cost of $\mathcal{O}(dn_t^2 + n_t^2\mathrm{log}(n_t))$. Generally, we have $C \ \textless \ d  \ \textless \ m$. Therefore, the overall computational complexity is $\mathcal{O}(n_t^2\mathrm{log}(n_t)(T+1) + n^2mT + nm^2T + dm^2T)$, where $T$ is the number of iteration.

\section{Semi-supervised Extension \label{SDA extension}}
In this section, we further extend our CMMS to semi-supervised domain adaptation including both homogeneous and heterogeneous settings.
\subsubsection{Homogeneous Setting}
We denote the target data as $\mathbf{X}_t = \{\mathbf{X}_l, \mathbf{X}_u\}$ where $\mathbf{X}_l = \{\mathbf{x}_{li}\}_{i=1}^{n_l}$ is the labeled data with the corresponding labels denoted by $\mathbf{Y}_l = \{\mathbf{y}_{li}\}_{i=1}^{n_l}$ and $\mathbf{X}_u = \{\mathbf{x}_{uj}\}_{j=1}^{n_u}$ is the unlabeled data. In the SDA scenario, except for the class centroids of source data,  the few but precise labeled target data can provide additional valuable reference for determining the cluster centroids $\mathbf{F}$ of unlabeled data. In this paper, we provide a simple but effective strategy to adaptively combine these two kinds of information. Specifically, our proposed semi-supervised extension is formulated as:
\begin{equation}
\label{E21}
\begin{split}
&\min_{\mathbf{P}, \mathbf{F}, \mathbf{G}_u, \mathbf{S}} \|\lambda_1\mathbf{P}^\mathrm{T}\mathbf{X}_s\mathbf{E}_s+\lambda_2\mathbf{P}^\mathrm{T}\mathbf{X}_l\mathbf{E}_l-\mathbf{F}\|_F^2+\beta \|\mathbf{P}\|_F^2\\
&+\alpha\|\mathbf{P}^\mathrm{T}\mathbf{X}_t-\mathbf{F}\mathbf{G}_t^\mathrm{T}\|_F^2+\gamma(\mathrm{tr}(\mathbf{P}^\mathrm{T}\mathbf{X}\mathbf{L}\mathbf{X}^\mathrm{T}\mathbf{P})+\delta\|\mathbf{S}\|_F^2)\\
& \quad s.t.\ \lambda_1 + \lambda_2 = 1, \lambda_1,\lambda_2 \ge 0, \ \mathbf{P}^\mathrm{T}\mathbf{XHX}^\mathrm{T}\mathbf{P} = \mathbf{I}_d, \\
& \quad \ \quad \mathbf{G}_u {\in \{0,1\}^{n_{t} \times C}},  \mathbf{S}\mathbf{1}_{n_t} = \mathbf{1}_{n_t}, 0 \leq S_{ij} \leq 1
\end{split}
\end{equation}
where $n_t = n_l + n_u$, $\lambda_1$, $\lambda_2$ are balanced factors and $\mathbf{G}_t = [\mathbf{G}_l;\mathbf{G}_u]$. $\mathbf{G}_l$ has the same definition with $\mathbf{G}_s$. Eq.(\ref{E21}) can be transformed to the standard formula as Eq.(\ref{E9}): 
\begin{equation}
\label{E22}
\begin{split}
&\min_{\mathbf{P}, \mathbf{F}, \mathbf{G}, \mathbf{S}}\|\mathbf{P}^\mathrm{T}\mathbf{X}\mathbf{E}-\mathbf{F}\|_F^2 + \alpha\|\mathbf{P}^\mathrm{T}\mathbf{X}\mathbf{V}-\mathbf{F}\mathbf{G}^\mathrm{T}\|_F^2\\
&\qquad \quad +\beta \|\mathbf{P}\|_F^2+\gamma (\mathrm{tr}(\mathbf{P}^\mathrm{T}\mathbf{XLX}^\mathrm{T}\mathbf{P})+\delta\|\mathbf{S}\|_F^2)\\
&s.t.\ \lambda_1 + \lambda_2 = 1, \lambda_1,\lambda_2 \ge 0, \ \mathbf{P}^\mathrm{T}\mathbf{XHX}^\mathrm{T}\mathbf{P} = \mathbf{I}_d, \\
&\ \ \ \ \ \ \mathbf{G} {\in \{0,1\}^{n \times C}}, \  \mathbf{S}\mathbf{1}_{n_t} = \mathbf{1}_{n_t}, 0 \leq S_{ij} \leq 1
\end{split}
\end{equation}
where $ \mathbf{E} = [\lambda_1\mathbf{E}_s; \lambda_2\mathbf{E}_l; \mathbf{0}_{n_u \times C}]$, $\mathbf{G} = [\mathbf{0}_{n_s \times C}; \mathbf{G}_t]$. In addition to the balanced factors $\lambda_1$ and $\lambda_2$, the other variables in Eq.(\ref{E22}) can be readily solved with our Algorithm \ref{alg1}. Since the objective function is convex with respect to $\lambda_1$ and $\lambda_2$, they can be solved easily with the closed-form solution: $\lambda_1 = \mathrm{max}(\mathrm{min}(\mathrm{tr}(\mathbf{J}\mathbf{M}^\mathrm{T})/\mathrm{tr}(\mathbf{J}^\mathrm{T}\mathbf{J}), 1), 0)$, $\lambda_2 = 1 - \lambda_1$, where $\mathbf{J} = \mathbf{P}^\mathrm{T}\mathbf{X}_s\mathbf{E}_s- \mathbf{P}^\mathrm{T}\mathbf{X}_l\mathbf{E}_l$, $\mathbf{M} =  \mathbf{F} - \mathbf{P}^\mathrm{T}\mathbf{X}_l\mathbf{E}_l.$ 
\subsubsection{Heterogeneous Setting}
In the heterogeneous setting, the source and target data usually own different feature dimensions. Our proposed Eq.(\ref{E21}) can be naturally extended to the heterogeneous manner, only by replacing the projection matrix $\mathbf{P}$ with two separate ones \cite{Hsieh2016}:
\begin{equation}
\label{E23}
\begin{aligned}
&\min_{\mathbf{P}, \mathbf{F}, \mathbf{G}_u, \mathbf{S}} \|\lambda_1\mathbf{P}_s^\mathrm{T}\mathbf{X}_s\mathbf{E}_s+\lambda_2\mathbf{P}_t^\mathrm{T}\mathbf{X}_l\mathbf{E}_l-\mathbf{F}\|_F^2 \\
&+\gamma(2\mathrm{tr}(\mathbf{P}_s^\mathrm{T}\mathbf{X}_s\mathbf{L_s}\mathbf{X_s}^\mathrm{T}\mathbf{P}_s)+2\mathrm{tr}(\mathbf{P}_t^\mathrm{T}\mathbf{X}_t\mathbf{L}_t\mathbf{X}_t^\mathrm{T}\mathbf{P}_t)+\\ 
&\delta\|\mathbf{S}\|_F^2)+\alpha\|\mathbf{P}_t^\mathrm{T}\mathbf{X}_t-\mathbf{F}\mathbf{G}_t^\mathrm{T}\|_F^2+\beta(\|\mathbf{P}_s\|_F^2+\|\mathbf{P}_t\|_F^2) \\
&s.t.\ \lambda_1 + \lambda_2 = 1, \lambda_1,\lambda_2 \ge 0, \ \mathbf{Z}\mathbf{H}\mathbf{Z}^\mathrm{T} = \mathbf{I}_d, \\
  & \qquad \mathbf{G}_t  {\in \{0,1\}^{n_{t} \times C}},  \mathbf{S}\mathbf{1}_{n_t} = \mathbf{1}_{n_t}, 0 \leq S_{ij} \leq 1
\end{aligned}
\end{equation}
where $\mathbf{Z} = [\mathbf{P}_s^\mathrm{T}\mathbf{X}_s,\mathbf{P}_t^\mathrm{T}\mathbf{X}_t]$ is the new feature representations of two domains with the same dimension. $\alpha$ controls the effect of $K$-means clustering for target data and a proper value helps to achieve a good clustering performance. $\gamma$ determines the influence of discriminative information preservation for source data and local manifold self-learning for target data, which is crucial to the equality improvement of class centroids for two domains. By defining $\mathbf{X} = \mathrm{diag}(\mathbf{X}_s,\mathbf{X}_t)$, $\mathbf{P} = [\mathbf{P}_s;\mathbf{P}_t]$, Eq.(\ref{E23}) can be transformed to the standard formula as Eq.(\ref{E22}), and thus can be solved with the same algorithm.

\section{Experiments \label{Experiments}}

In this section, we first describe all involved datasets. Next, the details of experimental setup including comparison methods in UDA and SDA scenarios, training protocol and parameter setting are given. Then, the experimental results in UDA scenario, ablation study, parameter sensitivity and convergence analysis are presented. Finally, we show the results in SDA scenario. The source code of this paper is available at \url{https://github.com/LeiTian-qj/CMMS/tree/master}.
\subsection{Datasets and Descriptions}
We apply our method to seven benchmark datasets which are widely used in domain adaptation. These datasets are represented with different kinds of features including Alexnet-FC$_7$, SURF, DeCAF$_6$, pixel, Resnet50 and BoW. Table \ref{datasets} shows the overall descriptions of these datasets. We will introduce them in detail as follows.

\emph{Office31} \cite{Saenko2010} contains 4,110 images of office objects in 31 categories from three domains: Amazon (A), DSLR (D) and Webcam (W). Amazon images are downloaded from the online merchants. The images from DSLR domain are captured by a digital SLR camera while those from Webcam domain by a web camera. We adopt the AlexNet-FC$_7$ features\footnote{ \url{https://github.com/VisionLearningGroup/CORAL/tree/master/dataset}} fine-tuned on source domain. Following \cite{Liang2018}, we have 6 cross-domain tasks, i.e., ``A$\rightarrow$D", ``A$\rightarrow$W", ..., ``W$\rightarrow$D".

\emph{Office-Caltech10} \cite{Gong2012} includes 2,533 images of objects in 10 shared classes between Office31 dataset and the Caltech256 (C) dataset. The Caltech256 dataset is a widely used benchmark for object recognition. We use the 800-dim SURF features\footnote{\url{http://boqinggong.info/assets/GFK.zip}} and 4,096-dim DeCAF$_6$ features\footnote{ \url{https://github.com/jindongwang/transferlearning/blob/master/data/} \label{JDW}} \cite{Donahue2014}. Following \cite{Long2013-JDA}, we construct 12 cross-domain tasks, i.e., ``A$\rightarrow$C", ``A$\rightarrow$D", ..., ``W$\rightarrow$D".

\emph{MSRC-VOC2007} \cite{Long2014-TJM} consists of two subsets: MSRC (M) and VOC2007 (V). It is constructed by selecting 1,269 images in MSRC and 1,530 images in VOC2007 which share 6 semantic categories: aeroplane, bicycle, bird, car, cow, sheep. We utilize the 256-dim pixel features\footnote{ \url{http://ise.thss.tsinghua.edu.cn/~mlong/}}. Finally, we establish 2 tasks, ``M$\rightarrow$V" and ``V$\rightarrow$M".

\emph{MNIST-USPS} is made up of two handwritten digit image datasets: MNIST (Mn) and USPS (Us). Following \cite{TPN}, we sample  2,000 images from MNIST and 1,800 images from USPS. Besides, we utilize 2 conv-layer LeNet pre-trained on labeled source data to extract CNN features. Two tasks are obtained, i.e., ``Mn$\rightarrow$Us" and ``Us$\rightarrow$Mn".

\emph{Office-Home} \cite{Venkateswara2017} involves 15,585 images of daily objects in 65 shared classes from four domains: Art (artistic depictions of objects, Ar), Clipart (collection of clipart images, Cl), Product (images of objects without background, Pr) and Real-World (images captured with a regular camera, Re). We use the 4,096-dim Resnet50 features\footnote{ \url{https://github.com/hellowangqian/domainadaptation-capls}} released by \cite{Wang2019-SPL}. Similarly, we obtain 12 tasks, i.e., ``Ar$\rightarrow$Cl", ``Ar$\rightarrow$Pr", ..., ``Re$\rightarrow$Pr".

\begin{table}[h]
	\caption{Statistics of the seven benchmark datasets}
	\label{datasets}
	\centering
	\scalebox{0.7}{%
		\begin{tabular}{lcccc}
			\toprule
			Dataset & Subsets (Abbr.) & Samples & Feature (Size) & Classes \\ \midrule
			\multirow{3}{*}{Office31} & Amazon (A) & 2,817 & \multirow{3}{*}{Alexnet-FC$_6$ (4,096)} & \multirow{3}{*}{31} \\
			& DSLR (D) & 498 &  &  \\
			& Webcam (W) & 795 &  &  \\ \hline
			\multirow{4}{*}{Office-Caltech10} & Amazon (A) & 958 & \multirow{4}{*}{\begin{tabular}[c]{@{}c@{}}SURF (800)\\ DeCAF$_6$ (4,096)\end{tabular}} & \multirow{4}{*}{10} \\
			& Caltech (W) & 1,123 &  &  \\
			& DSLR (D) & 157 &  &  \\
			& Webcam (W) & 295 &  &  \\ \hline
			\multirow{2}{*}{MSRC-VOC2007} & MSRC (M) & 1,269 & \multirow{2}{*}{Pixel (256)} & \multirow{2}{*}{6} \\
			& VOC2007 (V) & 1,530 &  &  \\ \hline
			\multirow{2}{*}{MNIST-USPS} & {MNIST (Mn)} & {2,000} & \multirow{2}{*}{{CNN features (2048)}} & \multirow{2}{*}{{10}} \\
			& {USPS (Us)} & 1,800 &  &  \\ \hline
			\multirow{4}{*}{Office-Home} & Art (Ar) & 2,427 & \multirow{4}{*}{Resnet50 (2,048)} & \multirow{4}{*}{65} \\
			& Clipart (Cl) & 4,365 &  &  \\
			& Product (Pr) & 4,439 &  &  \\
			& RealWorld (Re) & 4,357 &  &  \\ \hline
			\multirow{2}{*}{Visda2017} & train & \multicolumn{1}{l}{{152,397}} & \multirow{2}{*}{{Resnet50 (2,048)}} & \multirow{2}{*}{{12}} \\
			& {validation} & {55,388} &  &  \\ \hline
			\multirow{5}{*}{\begin{tabular}[c]{@{}l@{}}Multilingual Reuters\\ Collection\end{tabular}} & English & 18,758 & BoW (1,131) & \multirow{5}{*}{6} \\
			& French & 26,648 & BoW (1,230) &  \\
			& German & 29,953 & BoW (1,417) &  \\
			& Italian & 24,039 & BoW (1,041) &  \\
			& Spanish & 12,342 & BoW (807) &  \\ \hline
		\end{tabular}%
	}
\end{table}

\emph{Visda2017} \cite{Visda} is a large-scale dataset for synthetic images to real images cross-domain task, which is first presented in 2017 Visual Domain Adaptation Challenge. Following \cite{TPN}, we take the the training data (152,397 synthetic images) as source domain, and the the validation data (55,388 real images) as target domain. The source and target domains share 12 object categories, i.e., aeroplane, bicycle, bus, car, horse, knife, motorcycle, person, plant, skatebord, train and trunk. We employ Resnet50 model pre-trained on ImageNet in Pytorch to extracted features.

\emph{Multilingual Reuters Collection} \cite{Amini2009} is a cross-lingual text dataset with about 11,000 articles from six common classes in five languages: English, French, German, Italian, and Spanish. All articles are sampled by BoW features\footnote{ \url{http://archive.ics.uci.edu/ml/datasets/Reuters+RCV1+RCV2+Multilingual,+Multiview+Text+Categorization+Test+collection}} with TF-IDF. Then, they are processed by PCA for dimension reduction and the reduced dimensionality for English, French, German, Italian and Spanish are 1,131, 1,230, 1,417, 1,041 and 807, respectively. We pick the Spanish as the target and each of the rest as the source by turns. Eventually, we gain four tasks.

\subsection{Experimental Setup}
\subsubsection{Comparison Methods in UDA Scenario}
1-NN, SVM\textsuperscript{\ref{SVM}},
GFK \cite{Gong2012},
JDA \cite{Long2013-JDA},
CORAL \cite{Sun2016},
DICD \cite{Li2018}, 
JGSA \cite{Zhang2017},
DICE \cite{Liang2018},
MEDA \cite{Wang2018-MEDA},
SPL \cite{Wang2019-SPL},
MCS \cite{Liang2019},
EasyTL \cite{EasyTL},
CAPLS \cite{CAPLS},
TADA \cite{TADA},
ADDA \cite{ADDA},
the method of \cite{IADDA},
TPN \cite{TPN},
CADA-P \cite{CADA-P},
DCAN \cite{DCAN},
DREMA \cite{DRMEA},
GVB-GD \cite{GVB-GD},
RSDA-DANN \cite{RSDA-DANN},
DSAN \cite{DSAN},
SimNet \cite{SimNet},
the method of \cite{Wu et.al},
3CATN \cite{3CATN},
DTA \cite{DTA},
DM-ADA \cite{DM-ADA},
ALDA \cite{ALDA}.

\subsubsection{Comparison Methods in SDA Scenario}
SVM$_t$, SVM$_{st}$\textsuperscript{\ref{SVM}}, 
MMDT \cite{Hoffman2014},
DTMKL \cite{Duan2012},
CDLS \cite{Tsai2016},
ILS \cite{Herath2017},
TFMKL-S and TFMKL-H \cite{Wang2019},
SHFA \cite{Li-2013},
The method of Li \emph{et al}. \cite{Li2018-PA}.
SVM$_t$, SVM$_{st}$, MMDT, DTMKL-f, TFMKL-S and TFMKL-H are employed in the homogeneous setting while SVM$_t$, MMDT, SHFA, CDLS, the method of \cite{Li2018-PA} in the heterogeneous setting and CDSPP \cite{CDSPP}.

\begin{table*}[]
	\centering
	\caption{Classification accuracies (\%) on  Office31 dataset ($\alpha = 0.1, \beta = 0.1$)}
	\label{office31}
	\scalebox{0.75}{%
		\begin{tabular}{ccccccccccccccc}
			\toprule
			Task & 1-NN & SVM & GFK & JDA & CORAL & DICD & JGSA & DICE & MEDA & EasyTL & CAPLS & SPL & MCS & CMMS \\ \midrule
			A$\rightarrow$D & 59.8 & 58.8 & 61.8 & 65.5 & 65.7 & 66.5 & 69.5 & 67.5 & 69.5 & 65.3 & 69.9 & 69.1 & 71.9 & \textbf{72.9} \\
			A$\rightarrow$W & 56.4 & 58.9 & 58.9 & 70.6 & 64.3 & 73.3 & 70.4 & 71.9 & 69.9 & 66.0 & 68.2 & 69.9 & \textbf{75.1} & 74.7 \\
			D$\rightarrow$A & 38.1 & 48.8 & 45.7 & 53.7 & 48.5 & 56.3 & 56.6 & 57.8 & 58.0 & 50.5 & 56.5 & \textbf{62.5} & 58.8 & 60.7 \\
			D$\rightarrow$W & 94.7 & 95.7 & 96.4 & 98.2 & 96.1 & 96.9 & \textbf{98.2} & 97.2 & 94.0 & 93.3 & 98.0 & 97.7 & 96.7 & 97.6 \\
			W$\rightarrow$A & 39.8 & 47.0 & 45.5 & 52.1 & 48.2 & 55.9 & 54.2 & 60.0 & 56.0 & 50.6 & 58.6 & 58.2 & 57.2 & \textbf{60.3} \\
			W$\rightarrow$D & 98.4 & 98.2 & 99.6 & 99.2 & 99.8 & 99.4 & 99.2 & \textbf{100.0} & 96.8 & 97.2 & \textbf{100.0} & 99.6 & 99.4 & 99.6 \\ \midrule
			Average & 64.5 & 67.9 & 68.0 & 73.4 & 70.4 & 74.7 & 74.7 & 75.7 & 74.0 & 70.5 &75.2 & 76.2 & 76.5 & \textbf{77.6} \\ \bottomrule
		\end{tabular}%
	}
\end{table*}

\begin{table*}[]
	\centering
	\caption{Classification accuracies (\%) on Office-Caltech10 dataset with SURF features ($\alpha = 0.1, \beta = 0.2$)}
	\label{office_surf}
	\scalebox{0.75}{%
		\begin{tabular}{ccccccccccccccc}
			\toprule
			Task            & 1-NN & SVM  & GFK  & JDA  & CORAL         & DICD & JGSA          & DICE          & MEDA          & EasyTL       & CAPLS & SPL           & MCS  & CMMS          \\ \midrule
			A$\rightarrow$C & 26.0 & 35.6 & 41.0 & 39.4 & \textbf{45.1} & 42.4 & 41.5          & 42.7          & 43.9          & 42.3          & 43.0  & 41.2          & 44.1 & 39.4          \\
			A$\rightarrow$D & 25.5 & 36.3 & 40.7 & 39.5 & 39.5          & 38.9 & 47.1          & 49.7          & 45.9          & 48.4          & 40.8  & 44.6          & \textbf{55.4} & 53.5          \\
			A$\rightarrow$W & 29.8 & 31.9 & 41.4 & 38.0 & 44.4          & 45.1 & 45.8          & 52.2          & 53.2          & 43.1         & 42.0  & \textbf{58.0} & 40.3 & 56.3          \\
			C$\rightarrow$A & 23.7 & 42.9 & 40.2 & 44.8 & 54.3          & 47.3 & 51.5          & 50.2          & 56.5          & 52.6          & 53.1  & 53.3          & 53.9 & \textbf{61.0} \\
			C$\rightarrow$D & 25.5 & 33.8 & 40.0 & 45.2 & 36.3          & 49.7 & 45.9          & 51.0          & 50.3          & 50.6 & 41.4  & 41.4          & 46.5 & \textbf{51.0}          \\
			C$\rightarrow$W & 25.8 & 34.6 & 36.3 & 41.7 & 38.6          & 46.4 & 45.4          & 48.1          & 53.9          & 53.9         & 43.7  & \textbf{61.7} & 54.2 & \textbf{61.7} \\
			D$\rightarrow$A & 28.5 & 34.3 & 30.7 & 33.1 & 37.7          & 34.5 & 38.0          & 41.1          & 41.2          & 38.3         & 34.0  & 35.3          & 38.3 & \textbf{46.7} \\
			D$\rightarrow$C & 26.3 & 32.1 & 31.8 & 31.5 & 33.8          & 34.6 & 29.9          & 33.7          & 34.9          & \textbf{36.1}          & 27.1  & 25.9          & 31.6 & 31.9          \\
			D$\rightarrow$W & 63.4 & 78.0 & 87.9 & 89.5 & 84.7          & 91.2 & \textbf{91.9} & 84.1          & 87.5          & 86.1          & 82.7  & 82.7          & 85.4 & 86.1          \\
			W$\rightarrow$A & 23.0 & 37.5 & 30.1 & 32.8 & 35.9          & 34.1 & 39.9          & 37.5          & \textbf{42.7} & 38.2         & 36.3  & 41.1          & 37.3 & 40.1          \\
			W$\rightarrow$C & 19.9 & 33.9 & 32.0 & 31.2 & 33.7          & 33.6 & 33.2          & \textbf{37.8} & 34.0          & 35.4         & 33.1  & 37.8          & 33.8 & 35.8          \\
			W$\rightarrow$D & 59.2 & 80.9 & 84.4 & 89.2 & 86.6          & 89.8 & \textbf{90.5} & 87.3          & 88.5          & 79.6          & 80.3  & 83.4          & 80.9 & 89.2          \\ \midrule
			Average         & 31.4 & 42.6 & 44.7 & 46.3 & 47.6          & 49.0 & 50.0          & 51.3          & 52.7          & 50.5         & 46.5  & 50.5          & 50.1 & \textbf{54.4} \\ \bottomrule
		\end{tabular}%
	}
\end{table*}

\subsubsection{Training Protocol} For UDA scenario, all source samples are utilized for training like \cite{Li2018}. We exploit $z$-score standardization \cite{Gong2012} on all kinds of features. For SDA scenario, in homogeneous setting, we use the Office-Caltech10 and MSRC-VOC2007 datasets following the same protocol with \cite{Wang2019}. Specifically, for the Office-Caltech10 dataset, we randomly choose 20 samples per category for amazon domain while 8 for the others as the sources. Three labeled target samples per class are selected for training while the rest for testing. For fairness, we use the train/test splits released by \cite{Hoffman2014}. For the MSRC-VOC2007 dataset, all source samples are utilized for training, and 2 or 4 labeled target samples per category are randomly selected for training with the remaining to be recognized. In heterogeneous setting, we employ the Office-Caltech10 and Multilingual Reuters Collection datasets using the experiment setting of \cite{Li2018-PA}. For the Office-Caltech10 dataset, the SURF and DeCAF$_6$ features are served as the source and target. The source domain contains 20 instances per class, and 3 labeled target instances per category are selected for training with the rest for testing. For the Multilingual Reuters Collection dataset, Spanish is selected as the target and the remaining as the source by turns. 100 articles per category are randomly selected to build the source domain, and 10 labeled target articles per category are selected for training with 500 articles per class from the rest to be classified.

\begin{table*}[]
	\centering
	\caption{Classification accuracies (\%) on Office-Caltech10 dataset  with DeCAF$_6$ features ($\alpha = 0.2, \beta = 0.5$)}
	\label{office_decaf}
	\scalebox{0.75}{%
		\begin{tabular}{ccccccccccccccc}
			\toprule
			Task & 1-NN & SVM & GFK & JDA & CORAL & DICD & JGSA & DICE & MEDA & EasyTL & CAPLS & SPL & MCS & CMMS \\ \midrule
			A$\rightarrow$C & 71.7 & 84.4 & 77.3 & 83.2 & 83.2 & 86.0 & 84.9 & 85.9 & 87.4 & 86.5 & 86.1 & 87.4 & 88.3 & \textbf{88.8} \\
			A$\rightarrow$D & 73.9 & 83.4 & 84.7 & 86.6 & 84.1 & 83.4 & 88.5 & 89.8 & 88.1 & 91.7 & 94.9 & 89.2 & 91.7 & \textbf{95.5} \\
			A$\rightarrow$W & 68.1 & 76.9 & 81.0 & 80.3 & 74.6 & 81.4 & 81.0 & 86.4 & 88.1 & 85.8 & 87.1 & \textbf{95.3} & 91.5 & 92.2 \\
			C$\rightarrow$A & 87.3 & 91.3 & 88.5 & 88.7 & 92.0 & 91.0 & 91.4 & 92.3 & 93.4 & 93.0 & 90.8 & 92.7 & 93.5 & \textbf{94.1} \\
			C$\rightarrow$D & 79.6 & 85.4 & 86.0 & 91.1 & 84.7 & 93.6 & 93.6 & 93.6 & 91.1 & 89.2 & 95.5 & \textbf{98.7} & 90.4 & 95.5 \\
			C$\rightarrow$W & 72.5 & 77.3 & 80.3 & 87.8 & 80.0 & 92.2 & 86.8 & 93.6 & \textbf{95.6} & 82.7 & 85.4 & 93.2 & 85.1 & 91.9 \\
			D$\rightarrow$A & 49.9 & 86.5 & 85.8 & 91.8 & 85.5 & 92.2 & 92.0 & 92.5 & 93.2 & 91.3 & 93.0 & 92.9 & \textbf{93.5} & 93.4 \\
			D$\rightarrow$C & 42.0 & 77.1 & 76.0 & 85.5 & 76.8 & 86.1 & 86.2 & 87.4 & 87.5 & 84.5 & 88.8 & 88.6 & 89.0 & \textbf{89.3} \\
			D$\rightarrow$W & 91.5 & 99.3 & 97.3 & 99.3 & 99.3 & 99.0 & 99.7 & 90.0 & 97.6 & 98.0 & \textbf{100.0} & 98.6 & 99.3 & 99.3 \\
			W$\rightarrow$A & 62.5 & 80.7 & 81.8 & 90.2 & 81.2 & 89.7 & 90.7 & 90.7 & \textbf{99.4} & 89.8 & 92.3 & 92.0 & 93.4 & 93.8 \\
			W$\rightarrow$C & 55.3 & 72.5 & 73.9 & 84.2 & 75.5 & 84.0 & 85.0 & 85.3 & \textbf{93.2} & 80.8 & 88.2 & 87.0 & 88.3 & 89.0 \\
			W$\rightarrow$D & 98.1 & 99.4 & \textbf{100.0} & \textbf{100.0} & \textbf{100.0} & \textbf{100.0} & \textbf{100.0} & \textbf{100.0} & 99.4 & 99.4 & \textbf{100.0} & \textbf{100.0} & 100.0 & \textbf{100.0} \\ \midrule
			Average & 71.0 & 84.5 & 83.2 & 89.1 & 84.7 & 89.9 & 90.0 & 91.4 & 92.8 & 89.4 & 91.8 & 93.0 & 92.0 & \textbf{93.6} \\ \bottomrule
		\end{tabular}%
	}
\end{table*}

\begin{table*}[]
	\centering
	\caption{Classification accuracies (\%) on MSRC-VOC2007 dataset (top, $\alpha = 0.1, \beta = 0.05$) and MNIST-USPS dataset (below, $\alpha = 0.2, \beta = 0.05$)}
	\label{msrc}
	\scalebox{0.75}{%
		\begin{tabular}{p{0.7cm}<{\centering}p{0.7cm}<{\centering}p{0.7cm}<{\centering}p{0.7cm}<{\centering}p{0.7cm}<{\centering}p{0.7cm}<{\centering}p{0.7cm}<{\centering}p{0.7cm}<{\centering}p{0.7cm}<{\centering}p{0.7cm}<{\centering}p{0.7cm}<{\centering}p{0.7cm}<{\centering}p{0.7cm}<{\centering}p{1.45cm}<{\centering}p{0.7cm}<{\centering}}
			\toprule
			Task & 1-NN & SVM & GFK & JDA & CORAL & DICD & JGSA & DICE & MEDA & EasyTL & CAPLS & SPL & MCS & CMMS \\ \midrule
			M$\rightarrow$V & 35.5 & 35.6 & 34.7 & 30.4 & \textbf{38.4} & 32.4 & 35.2 & 33.1 & 35.3 & 33.1 & 32.5 & 34.7 & 31.8 & 31.8 \\
			V$\rightarrow$M & 47.2 & 51.8 & 48.9 & 44.8 & 54.9 & 47.8 & 47.5 & 46.3 & 60.1 & 58.3 & 49.2 & 63.8 & 66.5 & \textbf{79.1} \\ \midrule
			Average & 41.3 & 43.7 & 41.8 & 37.6 & 46.7 & 40.1 & 41.3 & 39.7 & 47.7 & 45.7 & 40.9 & 49.3 & 49.2 & \textbf{55.4} \\ \bottomrule \toprule
			Task & 1-NN & SVM & JDA & DICD & JGSA & DICE & MEDA & EasyTL & SPL & MCS & ADDA & TPN & Chadha  \cite{IADDA} & CMMS \\ \midrule
			Mn$\rightarrow$Us & 90.5 & 88.6 & 93.2 & 94.4 & 93.4 & 94.9 & 94.7 & 90.3 & \textbf{95.4} & 94.9 & 89.4 & 92.1 & 92.5 & \textbf{95.4} \\
			Us$\rightarrow$Mn & 88.8 & 86.6 & 92.0 & 92.2 & 91.3 & 92.0 & 93.6 & 90.5 & 94.1 & 93.8 & 90.1 & 94.1 & \textbf{96.7} & 94.3 \\ \midrule
			Average & 89.7 & 87.6 & 92.6 & 93.3 & 92.3 & 93.4 & 94.2 & 90.4 & 94.7 & 94.3 & 89.8 & 93.1 & 94.6 & \textbf{94.8} \\ \bottomrule
		\end{tabular}%
	}
\end{table*}

\begin{table*}[]
	\centering
	\caption{Classification accuracies (\%) on Office-Home dataset ($\alpha = 0.1, \beta = 0.01$). Deep learning methods are listed below CMMS. The highest accuracy of each cross-domain task with traditional/deep learning methods is boldfaced.}
	\label{officehome}
	\scalebox{0.73}{%
		\begin{tabular}{p{1.7cm}<{\centering}p{0.6cm}<{\centering}p{0.6cm}<{\centering}p{0.6cm}<{\centering}p{0.6cm}<{\centering}p{0.6cm}<{\centering}p{0.6cm}<{\centering}p{0.6cm}<{\centering}p{0.6cm}<{\centering}p{0.6cm}<{\centering}p{0.6cm}<{\centering}p{0.6cm}<{\centering}p{0.6cm}<{\centering}p{0.6cm}<{\centering}p{0.6cm}<{\centering}}
			\toprule
			Methods & \multicolumn{1}{l}{Ar$\rightarrow$Cl} & \multicolumn{1}{l}{Ar$\rightarrow$Pr} & \multicolumn{1}{l}{Ar$\rightarrow$Re} & \multicolumn{1}{l}{Cl$\rightarrow$Ar} & \multicolumn{1}{l}{Cl$\rightarrow$Pr} & \multicolumn{1}{l}{Cl$\rightarrow$Re} & \multicolumn{1}{l}{Pr$\rightarrow$Ar} & \multicolumn{1}{l}{Pr$\rightarrow$Cl} & \multicolumn{1}{l}{Pr$\rightarrow$Re} & \multicolumn{1}{l}{Re$\rightarrow$Ar} & \multicolumn{1}{l}{Re$\rightarrow$Cl} & \multicolumn{1}{l}{Re$\rightarrow$Pr} & \multicolumn{1}{l}{Average} \\ \midrule
			1-NN & 37.9 & 54.4 & 61.6 & 40.7 & 52.7 & 52.5 & 47.1 & 41.1 & 66.7 & 57.1 & 45.1 & 72.9 & 52.5 \\
			SVM & 42.4 & 61.2 & 69.9 & 42.6 & 56.2 & 57.7 & 48.7 & 41.5 & 70.6 & 61.6 & 45.7 & 76.1 & 56.2 \\
			GFK & 38.7 & 57.7 & 63.0 & 43.3 & 54.6 & 54.2 & 48.0 & 41.6 & 66.8 & 58.1 & 45.0 & 72.8 & 53.6 \\
			JDA & 45.8 & 63.6 & 67.5 & 53.3 & 62.2 & 62.9 & 56.0 & 47.1 & 72.9 & 61.8 & 50.5 & 75.2 & 59.9 \\
			CORAL & 47.3 & 69.3 & 74.6 & 54.2 & 67.2 & 67.8 & 55.7 & 43.0 & 73.9 & 64.2 & 49.2 & 78.0 & 62.0 \\
			DICD & 53.0 & 73.6 & 75.7 & 59.7 & 70.3 & 70.6 & 60.9 & 49.4 & 77.7 & 67.9 & 56.2 & 79.7 & 66.2 \\
			JGSA & 51.3 & 72.9 & 78.5 & 58.1 & 72.4 & 73.4 & 62.3 & 50.3 & 79.4 & 67.9 & 53.4 & 80.4 & 63.3 \\
			DICE & 49.1 & 70.7 & 73.9 & 51.4 & 65.9 & 65.9 & 60.0 & 48.6 & 76.2 & 65.4 & 53.5 & 78.8 & 63.3 \\
			MEDA & 52.1 & 75.3 & 77.6 & 61.0 & 76.5 & 76.8 & 61.8 & 53.4 & 79.5 & 68.1 & 55.1 & 82.5 & 68.3 \\
			EasyTL & 49.8 & 72.5 & 75.8 & 60.7 & 69.5 & 71.2 & 59.0 & 47.1 & 76.4 & 64.8 & 51.1 & 77.3 & 64.6 \\
			CAPLS & \textbf{56.2} & 78.3 & 80.2 & \textbf{66.0} & 75.4& 78.4 & 66.4 & 53.2 & 81.1 & \textbf{71.6} & 56.1 & 84.3 & 70.6 \\
			SPL & 54.5 & 77.8 & 81.9 & 65.1 & 78.0 & 81.1 & 66.0 & 53.1 & 82.8 & 69.9 & 55.3 & \textbf{86.0} & 71.0 \\
			MCS & 53.8 & 78.4 & 78.8 & 64.0 & 75.0 & 78.9 & 64.8 & 52.3 & 79.9 & 67.0 & 55.8 & 80.3 & 69.1 \\
			CMMS & \textbf{56.2} & \textbf{80.8} & \textbf{82.8} & 65.9 & \textbf{78.7} & \textbf{82.2} & \textbf{67.7} & \textbf{54.5} & \textbf{82.9} & 69.5 & \textbf{57.1} & 85.2 & \textbf{72.0} \\ \midrule
			TADA & 53.1 & 72.3 & 77.2 & 59.1 & 71.2 & 72.1 & 59.7 & 53.1 & 78.4 & 72.4 & 60.0 & 82.9 & 67.6 \\
			CADA-P & 56.9 & 76.4 & 80.7 & 61.3 & \textbf{75.2} & 75.2 & 63.2 & 54.5 & 80.7 & 73.9 & \textbf{61.5} & 84.1 & 70.2 \\
		    DCAN & 54.5 & 75.7 & \textbf{81.2} & \textbf{67.4} & 74.0 & 76.3 & \textbf{67.4} & 52.7 & 80.6 & 74.1 & 59.1 & 83.5 & \textbf{70.5} \\
			DRMEA & 52.3 & 73.0 & 77.3 & 64.3 & 72.0 & 71.8 & 63.6 & 52.7 & 78.5 & 72.0 & 57.7 & 51.6 & 68.1 \\
			GVB-GD & \textbf{57.0} & 74.7 & 79.8 & 64.6 & 74.1 & 74.6 & 65.2 & \textbf{55.1} & 81.0 & 74.6 & 59.7 & 84.3 & 70.4 \\
			RSDA-DANN & 51.5 & \textbf{76.8} & 81.1 & 67.1 & 72.1 & \textbf{77.0} & 64.2 & 51.1 & \textbf{81.8} & \textbf{74.9} & 55.9 & \textbf{84.5} & 69.8 \\
			DSAN & 54.4 & 70.8 & 75.4 & 60.4 & 67.8 & 68.0 & 62.6 & 55.9 & 78.5 & 73.8 & 60.6 & 83.1 & 67.6 \\ \bottomrule
		\end{tabular}%
	}
\end{table*}

\subsubsection{Parameter Setting}In both UDA and SDA scenarios, we do not have massive labeled target samples, so we can not perform a standard cross-validation procedure to obtain the optimal parameters. For a fair comparison, we cite the results from the original papers or run the code provided by the authors. Following \cite{Li2018}, we grid-search the hyper-parameter space and report the best results. For GFK, JDA, DICD, JGSA, DICE and MEDA, the optimal reduced dimension is searched in $d \in \{10,20,...,100\}$. The best value of regularization parameter for projection is searched in the range of $\{0.01,0.02,0.05,0.1,0.2,0.5,1.0\}$. For the two recent methods, SPL and MCS, we adopt the default parameters used in their public codes or follow the procedures for tuning parameters according to the corresponding original papers. For our method, we fix $d = 100$, $\gamma = 5.0$ and $k = 10$ leaving $\alpha$, $\beta$ tunable. We obtain the optimal parameters by searching $\alpha$, $\beta \in [0.01,0.02,0.05,0.1,0.2,0.5,1.0]$.

\subsection{Unsupervised Domain Adaptation}
\subsubsection{Experimental results on UDA}
\textbf{Results on Office31 dataset}. Table \ref{office31} summarizes the classification results on the Office31 dataset, where the highest accuracy of each cross-domain task is boldfaced. We can observe that CMMS has the best average performance, with a 1.1$\%$ improvement over the optimal competitor MCS. CMMS achieves the highest results on 2 out of 6 tasks, while MCS only works the best for task A$\rightarrow$W with just 0.4$\%$ higher than CMMS. Generally, SPL, MCS and CMMS perform better than those methods that classify target samples independently, which demonstrates that exploring the structure information of data distribution can facilitate classification performance. However, compared with SPL and MCS, CMMS further mines and exploits the inherent local manifold structure of target data to promote cluster prototypes learning, thus can lead to a better performance.

\textbf{Results on Office-Caltech10 dataset}. The results on  Office-Caltech10 dataset with SURF features are listed in Table \ref{office_surf}. Regarding the average accuracy, CMMS shows a large advantage which improves 1.7$\%$ over the second best method MEDA. CMMS is the best method on 4 out of 12 tasks, while MEDA only wins two tasks. On C$\rightarrow$A, D$\rightarrow$A and C$\rightarrow$W, CMMS leads MEDA by over 4.5$\%$ margin. Following \cite{Wang2019-SPL}, we also employ the DeCAF$_6$ features, and the classification results are shown in Table \ref{office_decaf}. CMMS is superior to all competitors with regard to the average accuracy and works the best or second best for all tasks except for C$\rightarrow$W. Carefully comparing the results of SURF features and DeCAF$_6$ features, we can find that SPL and MCS prefer to deep features. Nevertheless, CMMS does not have such a preference, which illustrates that CMMS owns better generalization capacity.

\textbf{Results on MSRC-VOC2007 and MNIST-USPS datasets}. The experimental results on the MSRC-VOC2007 and MNIST-USPS datasets are reported in Table \ref{msrc}. On MSRC-VOC2007 dataset, the average classification accuracy of CMMS is 55.4$\%$, which is significant higher than those of all competitors. Especially, on task V$\rightarrow$M, CMMS gains a huge performance improvement of 15.3$\%$ compared with the second best method SPL, which verifies the significant effectiveness of our proposal. On MNIST-USPS dataset, CMMS achieves the highest average classification accuracy among all methods, which confirms the superior generalization capacity of our proposal. 

\begin{table*}[]
	\centering
	\caption{Classification accuracies (\%) on Visda2017 dataset ($\alpha = 0.1, \beta = 0.01$). Deep learning methods are listed  below our CMMS. $^*$  indicates the results are obtained using  ResNet-101 features or ResNet-101 model, and $-$ represents the results are unavailable. The results of JGSA  and LPJT are cited from $\cite{Li2019-LPJT}$. The highest accuracy of each task with traditional/deep learning methods is boldfaced.}
	\label{visda}
	\scalebox{0.75}{%
		\begin{tabular}{cccccccccccccc}
			\toprule
			Method & aeroplane & bicycle & bus & car & horse & knife & motorcycle & person & plant & skateboard & train & truck & Mean \\ \midrule
			1-NN & 90.0 & 55.2 & 66.1 & 43.0 & 56.8 & 26.2 & \textbf{91.4} & 31.4 & 35.3 & 31.4 & 70.6 & 36.6 & 52.8 \\
			SVM & 77.1 & 41.8 & 73.5 & 53.3 & 53.2 & 21.5 & 82.9 & 34.9 & 57.2 & 51.6 & \textbf{83.5} & 24.9 & 50.2 \\
			GFK & 88.5 & 55.9 & 67.9 & 42.4 & 65.3 & 30.0 & 90.6 & 36.4 & 42.0 & 32.3 & 73.3 & 34.1 & 54.9 \\
			JDA & 91.2 & 73.2 & 67.4 & 52.5 & 88.9 & 67.6 & 85.1 & 69.0 & 80.7 & 48.1 & 81.6 & 31.9 & 69.8 \\
			JGSA$^*$ & 91.1 & 63.8 & 64.2 & 47.5 & 92.6 & 42.3 & 79.8 & 76.5 & 84.3 & 50.1 & 72.2 & 26.5 & 65.8 \\
			LPJT$^*$ & 93.0 & 80.3 & 66.5 & 56.3 & \textbf{95.8} & 70.3 & 74.2 & \textbf{83.8} & \textbf{91.7} & 40.0 & 78.7 & 57.6 & 74.0 \\
			MCS & \textbf{94.9} & \textbf{80.6} & 79.1 & 56.7 & 92.8 & 94.8 & 82.9 & 70.0 & 84.4 & \textbf{91.3} & 81.4 & 43.1 & 79.3 \\
			CMMS & 93.9 & 76.9 & \textbf{86.1} & \textbf{58.7} & 93.2 & \textbf{96.3} & 84.1 & 80.1 & 83.7 & 87.2 & 80.0 & \textbf{62.7} & \textbf{81.9} \\ \midrule
			SimNet & \textbf{94.5} & 80.2 & 69.5 & 43.5 & 89.5 & 16.6 & 76.0 & 81.1 & 86.4 & 76.4 & 79.6 & 41.9 & 69.6 \\
			Wu $\emph{et al.}$ \cite{Wu et.al} & 91.5 & 80.4 & 82.3 & 59.2 & 90.1 & 62.9 & 85.3 & 75.0 & 87.3 & 70.6 & 78.5 & 28.4 & 74.3 \\
			3CATN & $-$ & $-$ & $-$ & $-$ & $-$ & $-$ & $-$ & $-$ & $-$ & $-$ & $-$ & $-$ & 73.2 \\
			TPN & 93.7 & \textbf{85.1} & 69.2 & 81.6 & \textbf{93.5} & 61.9 & 89.3 & 81.4 & 93.5 & \textbf{81.6} & 84.5 & 49.9 & \textbf{80.4} \\
			DTA & 93.1 & 70.5 & \textbf{83.8} & \textbf{87.0} & 92.3 & 3.3 & 91.9 & \textbf{86.4} & 93.1 & 71.0 & 82.0 & 15.3 & 76.2 \\
			DM-ADA$^*$ & $-$ & $-$ & $-$ & $-$ & $-$ & $-$ & $-$ & $-$ & $-$ & $-$ & $-$ & $-$ & 75.6 \\
			ALDA & 87.0 & 61.3 & 78.7 & 67.9 & 83.7 & 89.4 & 89.5 & 71.0 & \textbf{95.4} & 71.9 & \textbf{89.6} & 33.1 & 76.5 \\
			GVB-GD & $-$ & $-$ & $-$ & $-$ & $-$ & $-$ & $-$ & $-$ & $-$ & $-$ & $-$ & $-$ & 75.3 \\ 
			RSDA-DANN & $-$ & $-$ & $-$ & $-$ & $-$ & $-$ & $-$ & $-$ & $-$ & $-$ & $-$ & $-$ & 75.8 \\
			DSAN$^*$ & 90.9 & 66.9 & 75.7 & 62.4 & 88.9 & \textbf{77.0} & \textbf{93.7} & 75.1 & 92.8 & 67.6 & 89.1 & 39.4 & 75.1 \\
			\bottomrule
		\end{tabular}%
	}
\end{table*}

\begin{table*}[]
	\centering
	\caption{The classification performance(\%) of CMMS, the five variants of CMMS and JDA.}
	\label{ablation_study}
	\scalebox{0.85}{%
		\begin{tabular}{p{3.5cm}<{\centering}p{1.1cm}<{\centering}p{1.1cm}<{\centering}p{1.1cm}<{\centering}p{1.1cm}<{\centering}p{1.1cm}<{\centering}p{1.1cm}<{\centering}p{1.1cm}<{\centering}p{1.1cm}<{\centering}}
			\toprule
			Dataset & JDA & CMMS$_{cm}$ & CMMS$_{rm}$ & CMMS$_{pa}$ & CMMS$_{op}$ & CMMS$_{ds}$ & CMMS \\ \midrule
			Office31 & 73.4 & 74.6 & 75.8 & 76.2 & 77.3 & 74.8 & \textbf{77.6} \\
			Office-Caltech10(SURF) & 46.3 & 51.4 & 52.4 & 53.2  & 53.4 & 51.8 & \textbf{54.4} \\
			Office-Caltech10 (DeCAF$_6$) & 89.1 & 92.8 & 92.9 & 93.2  & 93.5 & 92.3 & \textbf{93.6} \\
			MSRC-VOC2007 & 37.6 & 53.2 & 53.5 & {55.1}  & 55.1 & 54.9 & \textbf{55.4} \\
			MNIST-USPS & 92.6 & 93.9 & 93.8 & 94.4 & 94.4 & 93.6 & \textbf{94.8} \\
			Office-Home & 61.9 & 70.2 & 70.7 & 71.2 & 71.7 & 70.6 & \textbf{72.0} \\
			Visda2017 & 69.8 & 73.8 & 80.0 & 78.5 & 78.4 & 74.9 & \textbf{81.9} \\ \midrule
			Average & 67.2 & 72.8 & 74.2 & 74.5 & 74.8 & 73.3 & \textbf{75.7} \\ \bottomrule
		\end{tabular}%
	}
\end{table*}

\textbf{Results on Office-Home dataset}. For fairness, we employ the deep features recently released by \cite{Wang2019-SPL}, which are extracted using the Resnet50 model pre-trained on ImageNet. Table \ref{officehome} summarizes the classification accuracies. For traditional methods, CMMS outperforms the second best method SPL in average performance, and achieves the best performance on 9 out of all 12 tasks while SPL only works the best for task Re$\rightarrow$Pr with just 0.8$\%$ superiority to CMMS. This phenomenon shows that even the target samples are well clustered within the deep feature space, exploiting the inherent local manifold structure is still crucial to the improvement of the classification performance. Besides, we show the results of eight recent deep learning methods for comparison, which all adopt the Resnet50 model as the basic network structure. In terms of the average accuracy, CMMS performs better than all deep methods with 1.5$\%$ improvement at least, which validates the significant superiority of our proposal.

\textbf{Results on Visda2017 dataset}. Visda2017 is a large-scale dataset that is published recently and has been widely employed by various deep learning methods. Following \cite{TPN}, we take the per category classification accuracy of target data and the mean of accuracy over all 12 categories as evaluation metric. All experiments are conducted with MATLAB R2016b on a standard Linux server with Intel(R) Xeon(R) CPU E5-2678 v3 @ 2.50GHz CPU, 251GB RAM. The experimental results of six traditional methods and eleven deep methods are provided in Table \ref{visda}. The results of some other traditional methods are not reported due to the limited memory. We can observe that CMMS achieves the highest mean accuracy among all traditional and deep methods, which demonstrates that our CMMS has outstanding capability to deal with the domain adaptation tasks with more data samples.

\subsubsection{Ablation Study \label{Analytical Experiments}}
To understand our CMMS more deeply, we propose five variants of CMMS: \textbf{a)} CMMS$_{cm}$, only considers class Centroid Matching for two domains, i.e., the combination of Eq.(\ref{E2}), Eq.(\ref{E3}) and Eq.(\ref{E7}); \textbf{b)} CMMS$_{rm}$, does not utilize the local manifold structure of target data, i.e., removing Eq.(\ref{E4}) from our objective function Eq.(\ref{E8}); \textbf{c)} CMMS$_{pa}$, considers local manifold structure of target data by Predefining Adjacency matrix in the  original feature space, i.e., replacing Eq.(4) with the Laplacian regularization; \textbf{d)} CMMS$_{ds}$, exploits the Discriminative Structure information of target domain via assigning pseudo-labels to target data and then minimizing the intra-class scatter in the projected space like Eq.(\ref{E5}); \textbf{e)} CMMS$_{op}$, carries out manifold learning withOut Projection, {\it  i.e,} initializing target adjacency matrix in the original feature space using local manifold self-learning strategy and keeping it unchanged in the subsequent iterations. 
Table \ref{ablation_study} shows the results of CMMS and all variants. 
The results of classical JDA are also provided. Based on this table, we will analyze our approach in more detail as follows.

\textbf{Effectiveness of class centroid matching.} 
CMMS$_{cm}$ consistently precedes JDA on six datasets, which confirms the remarkable superiority of our proposal to the MMD based pioneering approach. By the class centroid matching strategy, we can make full use of the structure information of data distribution, thus target samples are supposed to present favourable cluster distribution.
To have a clear illustration, in Fig. \ref{clustering}, we display the $t$-SNE \cite{Maaten2008} visualization of the target features in the projected space on task V$\rightarrow$M of MSRC-VOC2007 dataset. We can observe that JDA features are mixed together while CMMS$_{cm}$ features are well-separated with cluster structure, which verifies the significant effectiveness of our class centroid matching strategy.
\begin{figure}[]
	\setlength{\abovecaptionskip}{0pt}
	\setlength{\belowcaptionskip}{0pt}
	\renewcommand{\figurename}{Figure}
	\centering
	\includegraphics[width=0.35\textwidth]{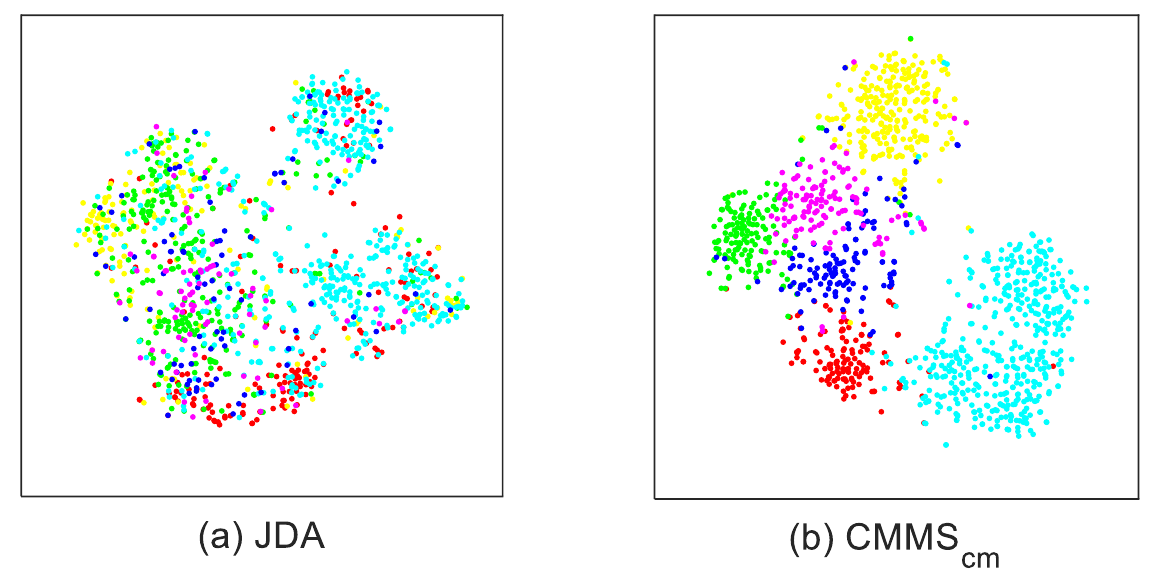}
	\caption{Target feature visualization of JDA and CMMS$_{cm}$ on task V$\rightarrow$M. Samples in different classes are presented by different colors.}
	\label{clustering}
\end{figure}

\textbf{Effectiveness of local manifold self-learning strategy for target data.} CMMS$_{pa}$ performs better than CMMS$_{rm}$ on nearly all datasets, which indicates that exploiting the local manifold structure of target samples help to classify them more successfully, even though the manifold structure is not so reliable. However, if we can capture it more faithfully, we can achieve a superior performance, which is verified by comparing CMMS with CMMS$_{pa}$. For a better understanding, we show the visualization of target adjacency matrix on task A$\rightarrow$D (SURF) in Fig. \ref{similarity}. These matrices are obtained by either the self-learned distance or the predefined distances which include Euclidean distance, heatkernel distance with kernel width 1.0 and cosine distance. 
As we can see from Fig. \ref{similarity}, all predefined distances
tend to incorrectly connect unrelated samples, and hardly capture the inherent local manifold structure of target data. However, the
self-learned distance can adaptively build the connections between intrinsic similar samples, thus can improve the classification performance. Moreover, CMMS $_{op}$ is inferior to CMMS on all datasets, which indicates the necessity of operating manifold learning in the projected space. In such way, the curse of dimensionality could be relieved, and the inherent local manifold structure of target data can be captured more effectively. Generally, CMMS$_{ds}$ performs much worse than CMMS and even worse than CMMS$_{rm}$ which verifies that utilizing the discriminative information of target domain via assigning pseudo-labels to target samples independently is far from enough to achieve satisfactory results.
The reason is that the pseudo-labels may be inaccurate and could cause error accumulation during learning, and thus the performance is degraded dramatically. In summary, our local manifold self-learning strategy can effectively enhance the utilization of structure information contained in target data.
\begin{figure}[h]
	\setlength{\abovecaptionskip}{0pt}
	\setlength{\belowcaptionskip}{0pt}
	\renewcommand{\figurename}{Figure}
	\centering
	\includegraphics[width=0.35\textwidth]{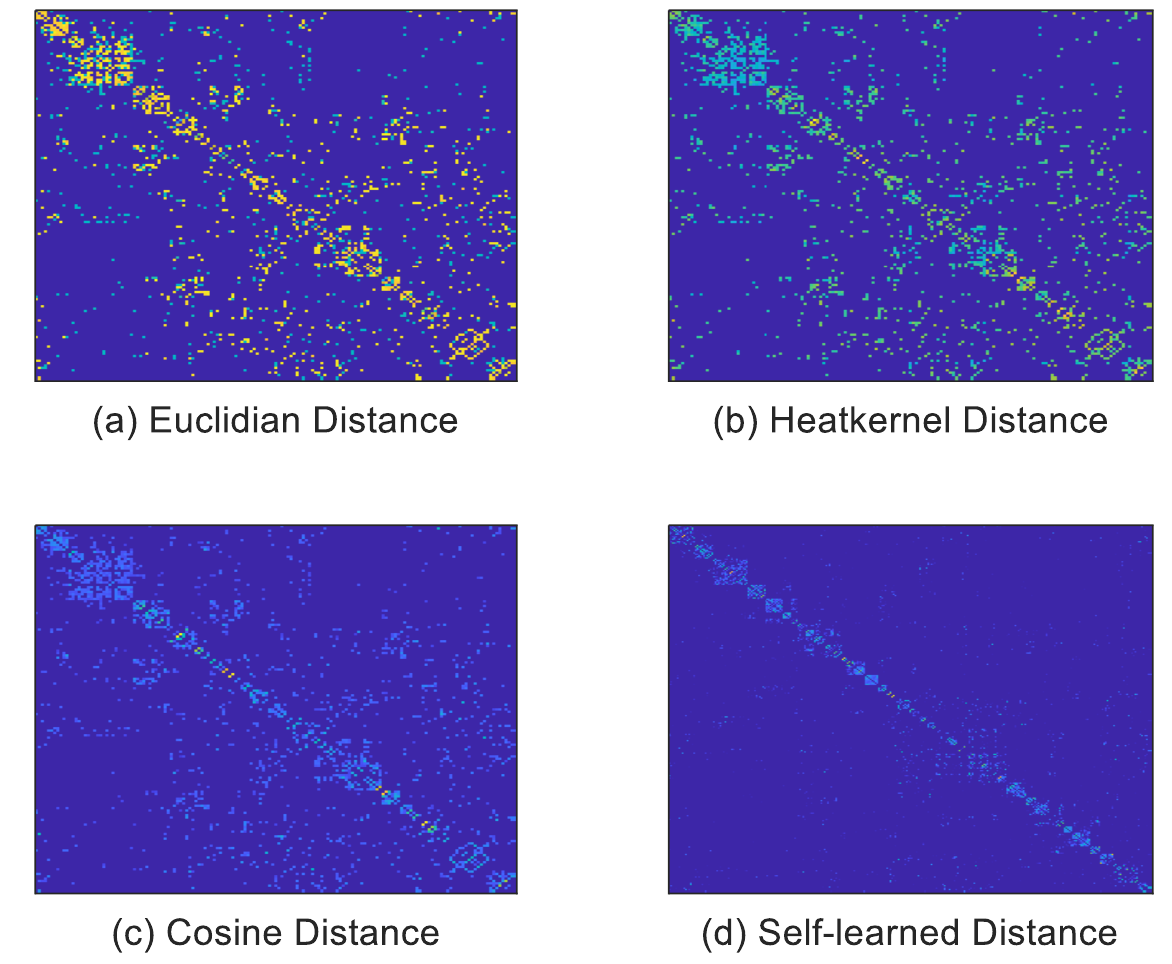}
	\caption{Target adjacency matrix visualization on task A$\rightarrow$D (SURF) obtained by predefined or self-learned distances.}
	\label{similarity}
\end{figure}

\subsubsection{Analytical Experiments}
In this part, we further conduct several experiments to pursue deeper insight for our CMMS approach. 

\textbf{Is the target prototype cluster assignment really equal to the target class?} 
To validate whether the target prototype cluster assignment is really equal to the actual target class, in Fig. \ref{r2} (a), we plot the evolution of classification error for each class given target ground-truth labels over training iterations on task V$\rightarrow$M.
As we can see, for most target prototype clusters (5 out of 6), the trend of classification error meets the ideal situation, {\it i.e.}, it would decrease over iterations. However, we find that the prototype cluster 3 is an exception. 
To figure out the reason, in Table \ref{m}, we summarize the sample size and proportion of each class for target data. For a more intuitive understanding, we further display the visualization result of target features in the projected space according to the prototype cluster pseudo-labels  (Fig. \ref{r2} (d)) and the ground-truth labels (Fig. \ref{r2} (e)). By the comprehensive analysis of Fig. \ref{r2} (d), (e) and Table \ref{m}, we argue that the potential reason could be that the sample size of target class 3 is very small and its sample distribution is easily mixed up with target class 6. In the training iteration process, more and more samples from target class 6 could be  divided into target prototype cluster 3, leading to the centroid of target prototype cluster 3 deviates from the actual centroid of target class 3.  One promising solution may be that during each training iteration, we can first include the target samples with high-confidence pseudo-label for each category to train by some selective strategies, e.g., self-paced learning, and then maximize the distances of different prototype clusters. 
\begin{table}[]
	\centering
	\caption{Size and proportion ($\%$) for each class of target data (V$\rightarrow$M).}
	\label{m}
	\scalebox{0.85}{%
		\begin{tabular}{ccccccc}
			\toprule
			Class & 1 & 2 & 3 & 4 & 5 & 6 \\ \midrule
			Size/Proportion & 58/4.6 & 272/21.4 & 72/5.7 & 495/39.0 & 182/14.3 & 190/15.0 \\ \bottomrule
		\end{tabular}%
	}
\end{table}

\textbf{Is there a semantic mismatch between the target prototype cluster and its nearest source class?} For a target prototype cluster, our CMMS assigns a pseudo-label $l_i^{cmms} (i=1,\ldots,C)$ to it by searching the nearest class centroid of source data.  
When the ground-truth labels are given, the true-label of a target prototype cluster could be determined by the majority voting of its member's label and here we call it the ``oracle label" (denoted as $l_i^{oracle}$).  In Fig. \ref{r2} (b), we plot two sets of curves. 
The solid lines display the distance between each target prototype cluster centroid and its corresponding real source class centroid ({\it i.e.,} the source class with $l_i^{oracle}$ label),
while the dashed lines display the distance between each target prototype cluster centroid and its nearest source class centroid ({\it i.e.,} the source class with $l_i^{cmms}$ label). 
Ideally, the solid line shall coincide with the dashed line for each target prototype cluster, which means an exact semantic match between the target prototype cluster and its nearest source class. From Fig. \ref{r2} (b), we observe that only target prototype cluster 3 semantically mismatches with source class 3. For the other 5 target prototype clusters, the solid lines overlap the dashed lines, and  the downward trend of these lines indicates that our CMMS could improve their corresponding semantic match over the training iterations.

\textbf{Does each target prototype cluster’s members become homogeneous over training progress?} In Fig. \ref{r2} (c), we display the evolution of homogeneity for each target prototype cluster on task V$\rightarrow$M.
It is obvious that all curves rise, which indicates that each target prototype cluster become homogeneous over training process. The homogeneity of 4 target clusters are more than 50$\%$, which means that our proposal can produce high-purity prototype clusters. The low homogeneity of target prototype cluster 3 further confirm our aforementioned analysis for  Fig. \ref{r2} (a), {\it i.e.,} the sample size of real target class 3 is very small and the samples of other target classes could be incorrectly divided into it.

\begin{figure*}[]
	\setlength{\abovecaptionskip}{0pt}
	\setlength{\belowcaptionskip}{0pt}
	\renewcommand{\figurename}{Figure}
	\centering
	\includegraphics[width=0.83\textwidth]{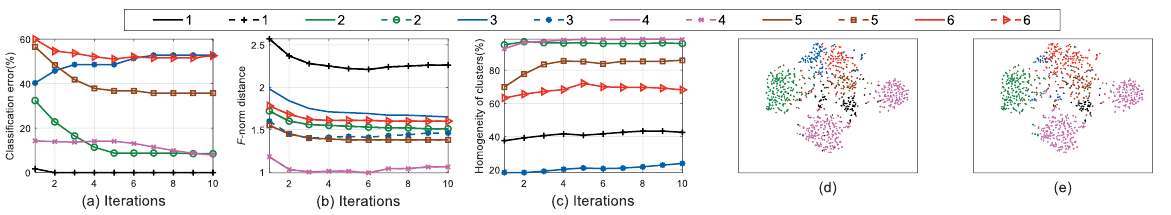}
	\caption{(a) Classification error given target ground-truth versus prototype cluster over training iterations on task V$\rightarrow$M. The number $i$ indicates the $i$th prototype cluster. (b) The $F$-norm distance between each prototype cluster centroid and its source data correspondence class centroid/nearest source class centroid over training iterations on task V$\rightarrow$M.  (c) The homogeneity of each cluster over training iterations on task V$\rightarrow$M. (d) Visualization of target features in the projected space with pseudo-labels. (e) Visualization of target features in the projected space with ground-truth labels.}
	\label{r2}
\end{figure*}

\begin{figure*}[]
	\setlength{\abovecaptionskip}{0pt}
	\setlength{\belowcaptionskip}{0pt}
	\renewcommand{\figurename}{Figure}
	\centering
	\includegraphics[width=0.85\textwidth]{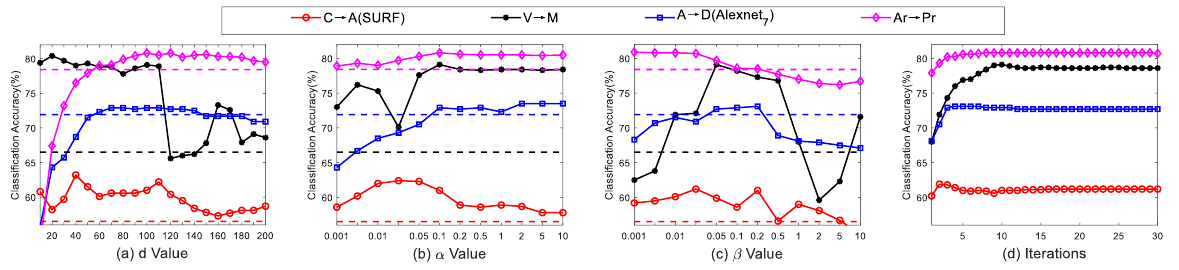}
	\caption{Parameter sensitivity analysis with respect to $d$, $\alpha$, $\beta$ and convergence analysis.}
	\label{parameter_converge}
\end{figure*}

\textbf{How does our CMMS perform on imbalanced dataset?} To evaluate the generalization ability of our CMMS on imbalanced dataset, an additional experiment is conducted on a new dataset. The source domain of this new dataset consists of Caltech10, dslr and webcam of Office-Caltech10 dataset. For its target domain, we design a new amazon domain. Specifically, for each category of the original amazon domain, we extract their first 82 images. The purpose of such operation is to force the sample proportion of each class to be equal in the new amazon domain. 
\begin{table}[h]
	\centering
	\caption{Mean of average classification accuracies (\%) for data imbalance issue (r: minimum retention rate).}
	\label{msrc_s}
	\scalebox{0.85}{%
		\begin{tabular}{cccccccccc}
			\toprule
			r  & 0.1 & 0.2 & 0.3 & 0.4 & 0.5 & 0.6 & 0.7 & 0.8 & 0.9 \\ \midrule
			EasyTL  & 32.6 & 34.1 & 36.2 & 38.4 & 39.6 & 40.0 & 42.2 & 42.5 & 43.1 \\
			CAPLS  & 32.3 & 34.5 & 35.9 & 38.3 & 39.2 & 39.3 & 41.1 & 41.6 & 41.9 \\
			MEDA  & 32.5 & 34.5 & 37.7 & 40.2 & 41.7 & 42.8 & 44.2 & 45.5 & 46.8 \\
			MCS  & 32.9 & 36.3 & 39.4 & 40.1 & 41.1 & 41.9 & 43.0 & 42.8 & 44.1 \\
			SPL  & 34.1 & 37.7 & 38.0 & 40.4& 40.9 & 42.5 & 43.3 & 43.1 & 45.3 \\
			CMMS  & \textbf{35.6} & \textbf{38.0} &\textbf{39.5} & \textbf{42.3}& \textbf{43.6} & \textbf{44.7} & \textbf{46.8} & \textbf{47.8} & \textbf{48.0} \\ \bottomrule
		\end{tabular}%
	}
\end{table}
We choose the new amazon domain as target and the rest as sources by turns, {\it i.e.}, we have three cross-domain tasks. And then, similar to the setting of \cite{DSEC} to study the performance of clustering algorithm on imbalanced datasets, we randomly sample nine subsets from the new dataset with different minimum retention rates. For the minimum retention rate $r$, the samples of the first category for target data will be kept with probability $r$, the last category with probability 1, and the other categories linearly in between.
We summarize the results in Table \ref{msrc_s}, where we report the mean of average classification accuracy on three tasks over five times. We  observe that data imbalance issue has negative influence on all methods. However, our CMMS could still obtain the highest accuracy with various minimum retention rates, which illustrates the advantage of our method on imbalanced dataset.
One of the potential reasons could be that the local manifold self-learning strategy introduced by our CMMS can capture more inherent structure of target data, thus making the influence of the datasets imbalance reduced. Even so, this insight inspires us to explore better solution of such issue in our future work.

\begin{table*}[h]
	\centering
	\caption{Classification accuracies (\%) on Office-Caltech10 dataset with SURF features for semi-supervised domain adaptation}
	\label{office_sda}
	\scalebox{0.75}{%
		\begin{tabular}{ccccccccccc}
			\toprule
			Task & SVM$_{t}$ & SVM$_{st}$ & MMDT & DTMKL-f & CDLS & ILS & TFMKL-S & TFMKL-H  & CMMS \\ \midrule
			A$\rightarrow$C & 31.1 & 42.4 & 36.4 & 40.5 & 37.1 & 43.6 & \textbf{43.8} & 43.5 & 43.0 \\
			A$\rightarrow$D & 56.9 & 47.7 & 56.7 & 45.9 & 61.9 & 49.8 & \textbf{62.0} & 57.3 & 61.8 \\
			A$\rightarrow$W & 62.8 & 50.1 & 64.6 & 47.9 & 69.3 & 59.7 & 70.9 & \textbf{71.4} & 69.2 \\
			C$\rightarrow$A & 44.7 & 47.2 & 49.4 & 47.3 & 52.5 & 55.1 & 54.2 & 54.4 &  \textbf{57.7} \\
			C$\rightarrow$D & 56.3 & 52.0 & 56.5 & 52.2 & 59.8 & 56.2 & 60.1 & \textbf{60.4} & 59.1 \\
			C$\rightarrow$W & 60.0 & 54.5 & 63.8 & 54.4 & \textbf{68.7} & 62.9 & 68.1 & 68.6 & 67.2 \\
			D$\rightarrow$A & 44.7 & 44.3 & 46.9 & 41.6 & 51.8 & 55.0 & 53.1 & 50.8 &  \textbf{55.2} \\
			D$\rightarrow$C & 31.3 & 36.8 & 34.1 & 36.0 & 36.9 & \textbf{41.0} & 38.9 & 37.9 & 39.5 \\
			D$\rightarrow$W & 62.0 & 80.6 & 74.1 & 77.6 & 70.7 & 80.1 & 79.1 & 76.7 & \textbf{82.5}\\
			W$\rightarrow$A & 45.4 & 45.2 & 47.7 & 45.3 & 52.3 & 54.3 & \textbf{54.4} & 54.0 & 53.4 \\
			W$\rightarrow$C & 29.7 & 36.1 & 32.2 & 36.3 & 35.1 & \textbf{38.6} & 36.2 & 34.9 & 37.7 \\
			W$\rightarrow$D & 56.5 & 71.2 & 67.0 & 69.6 & 61.3 & 70.8 & 69.1 & 69.3 & \textbf{72.9} \\ \midrule
			Average & 48.4 & 50.7 & 52.5 & 49.6 & 54.8 & 55.6 & 57.5 & 56.6 &  \textbf{58.3} \\ \bottomrule
		\end{tabular}%
	}
\end{table*}
\begin{table*}[h]
	\centering
	\caption{Classification accuracies (\%) on MSRC-VOC2007 dataset for semi-supervised domain adaptation}
	\label{msrc_sda}
	\scalebox{0.70}{%
		\begin{tabular}{cccccccccccc}
			\toprule
			& Task & SVM$_{t}$ & SVM$_{st}$ & MMDT & DTMKL-f & CDLS & ILS & TFMKL-S & TFMKL-H & CMMS \\ \midrule
			\multicolumn{1}{c|}{\multirow{3}{*}{n$_l$ = 2}} & M$\rightarrow$V & 28.9 & 38.5 & 35.1 & 36.8 & 30.2 & 34.2 & \textbf{38.2} & 36.6 & 35.8 \\
			\multicolumn{1}{c|}{} & V$\rightarrow$M & 55.9 & 55.5 & 59.9 & 64.1 & 55.8 & 49.9 & 68.5 & 71.0 & \textbf{76.8} \\ \cmidrule{2-11}
			\multicolumn{1}{c|}{} & Average & 42.4 & 47.0 & 47.5 & 50.5 & 43.0 & 42.1 & 53.4 & 53.8 & \textbf{56.3} \\ \midrule
			\multicolumn{1}{c|}{\multirow{3}{*}{n$_l$ = 4}} & M$\rightarrow$V & 30.2 & 39.0 & 36.0 & 36.9 & 31.7 & 35.4 & \textbf{38.4} & 36.8 & 36.2 \\
			\multicolumn{1}{c|}{} & V$\rightarrow$M & 64.4 & 56.6 & 62.1 & 65.0 & 57.3 & 50.2 & 70.4 & 71.4 & \textbf{77.7} \\ \cmidrule{2-11}
			\multicolumn{1}{c|}{} & Average & 47.3 & 47.8 & 49.1 & 51.0 & 44.5 & 42.8 & 54.4 & 54.1 & \textbf{57.0} \\ \bottomrule
		\end{tabular}%
	}
\end{table*}

\subsubsection{Parameters Sensitivity and Convergence Analysis}
In our CMMS, there are three tunable parameters: $\alpha$, $\beta$ and  $d$.  We  conducted extensive parameter sensitivity analysis with respect to the three parameters in a wide range. We vary one parameter once and fix the others as the optimal values. The results of C$\rightarrow$A (SURF), V$\rightarrow$M, A$\rightarrow$D (Alexnet$_7$) and Ar$\rightarrow$ Pr are reported in Fig. \ref{parameter_converge} (a) $\sim$ (c). Meanwhile, to demonstrate the effectiveness of our CMMS, we also display the results of the best competitor as the dashed lines.

First, we investigate the sensitivity of $d$ in a wide range $[10,20,...,200]$. As depicted in Fig. \ref{parameter_converge} (a), the optimal value of $d$ varies on different datasets. Nevertheless, our CMMS can consistently achieve better performance within a wide range $d \in [90,110]$. Note that, all the experimental results of our CMMS reported in this paper are achieved by fixing $d$ to 100, which demonstrates the potential of performance improvements of our approach. Then, we run our CMMS as $\alpha$ varies from 0.001 to 10.0. From Fig. \ref{parameter_converge} (a), it is observed that when the value of $\alpha$ is much small, it may not be contributed to the improvement of performance. Whereas, with the appropriate increase of $\alpha$, the clustering process of target data is emphasized, and thus our CMMS can exploit the cluster structure information more effectively. We find that, when $\alpha$ is located within a wide range $[0.1, 10.0]$, our proposal can consistently achieve optimal performance. Next, we evaluate the influence of parameter $\beta$ on our CMMS by varying the value from 0.001 to 10.0. It is infeasible to determine the optimal value of $\beta$, since it highly depends on the domain prior knowledge of the datasets. However, we empirically find that, when $\beta$ is located within the range $[0.05, 0.2]$, our CMMS can obtain better classification results than the most competitive competitor.
Furthermore, we  observe that in Fig. \ref{parameter_converge} (a) $\sim$ (c), the curves of task V$\rightarrow$M show  higher fluctuations compared to other tasks. The potential reason may be that the data imbalance on task V$\rightarrow$M makes our CMMS more sensitive to the setting of parameter values. Despite this, within the reasonable range of three parameter settings discussed above, our proposal can still significantly outperform the rivals. We also display the convergence analysis in Fig. \ref{parameter_converge} (d). We can see that CMMS can quickly converge within 10 iterations.

\begin{table}[h]
	\centering
	\caption{Classification accuracies(\%) on Office-Caltech10 dataset  for SURF $\rightarrow$ DeCAF$_6$}
	\label{office_surf_decaf}
	\scalebox{0.75}{%
		\begin{tabular}{cccccccc}
			\toprule
			Task & SVM$_t$ & MMDT & SHFA & CDLS & Li \cite{Li2018-PA} & CDSPP & CMMS \\ \midrule
			A$\rightarrow$C & 79.8 & 78.1 & 79.5 & 83.8 & 84.3 & 87.9 & \textbf{88.6} \\
			A$\rightarrow$W & 90.5 & 89.4 & 90.1 & 93.6 & 93.2 & \textbf{94.9} & 93.2 \\
			C$\rightarrow$A & 89.0 & 87.5 & 88.6 & 90.7 & 92.6 & 92.5 & \textbf{93.2} \\
			C$\rightarrow$W & 90.5 & 88.9 & 89.6 & 92.5 & 92.1 & \textbf{94.0} & 93.2 \\
			W$\rightarrow$A & 89.0 & 88.3 & 88.7 & 90.5 & \textbf{93.9} & 92.3 & 93.3 \\
			W$\rightarrow$C & 79.8 & 78.6 & 79.7 & 82.1 & 84.9 & 88.1 & \textbf{88.5} \\ \midrule
			Average & 86.4 & 85.1 & 86.0 & 88.9 & 90.2 & 91.6 & \textbf{91.7} \\ \bottomrule
		\end{tabular}%
	}
\end{table}
\begin{table}[h]
	\centering
	\caption{Classification accuracies(\%) on Multilingual Reuters Collection dataset (Spanish as the target domain)}
	\label{text}
	\scalebox{0.75}{%
		\begin{tabular}{cccccccc}
			\toprule
			Source & SVM$_t$ & MMDT & SHFA & CDLS & Li  \cite{Li2018-PA} & CDSPP & CMMS \\ \midrule
			English & \multirow{4}{*}{67.4} & 67.8 & 68.9 & 70.8 & 71.1 & 69.1 & \textbf{74.7} \\
			French &  & 68.3 & 69.1 & 71.2 & 71.2 & 69.0 & \textbf{74.4} \\
			German &  & 67.7 & 68.3 & 71.0 & 70.9 & 68.8 & \textbf{74.7} \\
			Italian &  & 66.5 & 67.5 & 71.7 & 71.5 & 68.8 & \textbf{74.6} \\ \midrule
			Average & 67.4 & 67.6 & 68.5 & 71.2 & 71.2 & 68.9 & \textbf{74.6} \\ \bottomrule
		\end{tabular}%
	}
\end{table}

\subsection{Semi-supervised Domain Adaptation}
\subsubsection{Results in Homogeneous Setting} The averaged classification results of all methods on the Office-Caltech10 dataset over 20 random splits are shown in Table \ref{office_sda}. We can observe that regarding the total average accuracy, our CMMS can obtain $0.8 \%$ improvement compared with the second best method TFMKL-S.
The results on the MSRC-VOC2007 dataset over 5 random splits are shown in Table \ref{msrc_sda}. Some results are cited from \cite{Wang2019}.
Compared with the most comparative competitors, CMMS can achieve $2.5 \%$ and $2.6 \%$ improvement when the number of labeled target samples in per class is set to 2 and 4, respectively.

\subsubsection{Results in Heterogeneous Setting} The results on the Office-Caltech10 and Multilingual Reuters Collection datasets are listed in Table \ref{office_surf_decaf} and Table \ref{text}. Some results are cited from \cite{Li2018-PA}. We can observe that our CMMS  achieves the optimal performance in terms of the average accuracy on both datasets. Specifically, compared with the best competitors, $0.1 \%$ and $3.4 \%$ improvements are obtained. Especially, CMMS works the best for 7 out of all 10 tasks on two datasets, which adequately confirms the excellent generalization capacity of our CMMS in the heterogeneous setting.

\section{Conclusions and Future Work \label{Conclusions}}
In this paper,  a novel domain adaptation method named CMMS is proposed. 
Unlike most of existing methods that generally assign pseudo-labels to target data independently, CMMS makes label prediction for target samples by the class centroid matching of source and target domains, such that the data distribution structure of target domain can be exploited. To explore the structure information of target data more thoroughly, a local manifold self-learning strategy is further introduced into CMMS, which can capture the inherent local manifold structure of target data by adaptively learning the data similarity in the projected space. The CMMS optimization problem is not convex with all variables, and thus an iterative optimization algorithm is designed to solve it, whose computational complexity and convergence are carefully analyzed. 
We further extend CMMS to the semi-supervised scenario including both homogeneous and heterogeneous settings which are appealing and promising. Extensive experimental results on seven datasets reveal that CMMS significantly outperforms the baselines and several state-of-the-art methods in both unsupervised and semi-supervised scenarios. 

Future research will include the following: 1) Considering the computational bottleneck of CMMS optimization, we will design more efficient algorithm for the local manifold self-learning; 2) Except for the class centroids, we can introduce additional measure to represent the structure of data distribution, such as the covariance; 3) In this paper, we extend CMMS to the semi-supervised scenario in a direct but effective way. In the future, more elaborate design of semi-supervised methods is worth further exploration; 4) It would be a promising research direction to design a systematical way to determine the optimal value of $d$.

\section*{Acknowledgment}
The authors would like to thank the editors and the anonymous reviewers for their constructive comments which helped to improve the quality of this article. The authors are thankful for the financial support by the the Key-Area Research and Development Program of Guangdong Province 2019B010153002 and the  National Natural Science Foundation of China (U1636220, 61961160707 and 61772525).

\ifCLASSOPTIONcaptionsoff
\newpage
\fi

\end{document}